\documentclass{article}

\usepackage[preprint,nonatbib]{neurips_2021}

\usepackage[utf8]{inputenc} 
\usepackage[T1]{fontenc}    
\usepackage[colorlinks = true, linkcolor=black]{hyperref}       
\usepackage{url}           
\usepackage{booktabs}       
\usepackage{amsfonts}       
\usepackage{nicefrac}       
\usepackage{microtype}     
\usepackage{xcolor,varwidth}         

\usepackage{graphbox}
\usepackage{amsmath}
\usepackage{amssymb}
\usepackage{caption}
\usepackage{makecell}
\usepackage{capt-of}
\usepackage{multirow}
\usepackage{booktabs} 

\usepackage{mathtools} 
\usepackage{amsthm}
\usepackage[labelformat=simple]{subcaption}
\usepackage{wrapfig, blindtext}
\usepackage{tabularx}
\usepackage{rotating}

\usepackage[toc,page,header]{appendix}
\usepackage{minitoc}

\makeatletter
\newcommand{\printfnsymbol}[1]{%
  \textsuperscript{\@fnsymbol{#1}}%
}
\makeatother

\usepackage{hyperref}

\newtheorem{proposition}{Proposition}

\newcommand{\ie}{\textit{i.e.}}
\newcommand{\eg}{\textit{e.g.} }

\title{Meta Internal Learning}

\author{
  Raphael Bensadoun, Shir Gur, Tomer Galanti, Lior Wolf\\
  The School of Computer Science, Tel Aviv University\\
  
 }
\begin{document}

\maketitle

\doparttoc 
\faketableofcontents 
\part{} 

\begin{abstract}
Internal learning for single-image generation is a framework, where a generator is trained to produce novel images based on a single image. Since these models are trained on a single image, they are limited in their scale and application. To overcome these issues, we propose a meta-learning approach that enables training over a collection of images, in order to model the internal statistics of the sample image more effectively.
In the presented meta-learning approach, a single-image GAN model is generated given an input image, via a convolutional feedforward hypernetwork $f$. This network is trained over a dataset of images, allowing for feature sharing among different models, and for interpolation in the space of generative models. The generated single-image model contains a hierarchy of multiple generators and discriminators. It is therefore required to train the meta-learner in an adversarial manner, which requires careful design choices that we justify by a theoretical analysis. Our results show that the models obtained are as suitable as single-image GANs for many common image applications, {significantly reduce the training time per image without loss in performance}, and introduce novel capabilities, such as interpolation and feedforward modeling of novel images. 
Our code is available at: \url{https://github.com/RaphaelBensTAU/MetaInternalLearning}.
\end{abstract}

\section{Introduction}
{In the field of internal learning, one wishes to learn the internal statistics of a signal in order to perform various downstream tasks. In this work, we focus on Single image GANs~\cite{shaham2019singan,shocher2018ingan,hinz2021improved,gur2020hierarchical}, which present extremely impressive results in modeling the distribution of images that are similar to the input image, and in applying this distribution to a variety of applications.} However, given that there is no shortage of unlabeled images, one may ask whether a better approach would be to model multiple images and only then condition the model on a single input image. 
Doing so, one could (i) benefit from knowledge and feature sharing between the different images, (ii) {better define} the boundaries between the distribution obtained from the input image and those of other images, (iii) possibly avoid the costly training phase given a novel image, and instead employ feedforward inference, and (iv) mix different single-image models to create novel types of images. 

From the algorithmic standpoint, this multi-image capability can be attempted using various forms of conditioning. For example, one can add a one-hot vector as an input, or, more generally, a vector signature, and train multiple images using the same single image method. One can also add a complete layer of a conditioning signal to the RGB input. Alternatively, one can employ StyleGAN-like conditioning and modify the normalization of the layers~\cite{karras2019style}. More generally, observing that this scenario is a meta-learning problem, one can employ methods, such as MAML~\cite{finn2017model} for learning a central network and its per-image variants. After performing many such attempts over a long period of time, we were not able to bring any of these methods to a desirable level of performance. 

Instead, we advocate for a meta-learning solution that is based on the hypernetworks scheme~\cite{ha2016hypernetworks}. Hypernetworks consist of two main components: a primary network $g$ that performs the actual computation, and the hypernetwork $f$ that is used for conditioning. The parameters (weights) of $g$ are not learned conventionally. Instead, they are given as the output of $f$ given the conditioned input signal.
Following a single-image GAN setting with a hierarchical structure, we have two hypernetworks $f_g$ and $f_d$, which dynamically produce the weights of the multiple generators and multiple discriminators, given the input image $I$. 

Our method allows for training on multiple images at once, obtaining similar results for various applications previously demonstrated for single image training. It also allows us to interpolate between single-image GANs derived from pairs of images (or more). Finally, we are able to fit a new unseen image in a fraction of the time that is required for training a new single image GAN, \ie our method enables inference generation for a novel image.

Since we are the first method, as far as we can ascertain, to perform adversarial training with hypernetworks, we provide a theoretical analysis of the proper way to perform this. It shows both the sufficiency of our algorithm for minimizing the objective function as well as the necessity of various components in our method. 

\section{Background}\label{sec:back}

In this paper, we consider the meta-learning problem of learning to generate a variety of samples from a single image, where each individual learning problem is defined by this single image input. For this purpose, we first recall the setting of single-image generation as in~\cite{shaham2019singan,hinz2021improved,gur2020hierarchical}.

\subsection{Single-Image Generation}
\label{sec:single_image_generation}
We start by describing single-image generation as introduced in SinGAN~\cite{shaham2019singan}. SinGAN is composed of a multi-scale residual generator $G = \{g_1,\dots,g_k\}$ and a patch-discriminator $D = \{d_1,\dots,d_k\}$, where $g_i$ and $d_i$ are fully-convolutional networks, consisting of five layers, and used for the training at scale $i$.
{Given an image $I$, we pre-compute $k$ scales of the image, from coarsest to finest, denoted by $I_i$, with height and width $h_i$ and $w_i$, and use each $I_i$ for the training the $i$'th generator $g_i$.}

The first generator $g_1$ takes as input a fixed random noise $z_1 \in \mathbb{R}^{3 \times h_1 \times w_1}$ whose coordinates are i.i.d. normally distributed, and outputs an image $\hat{I}_1 \in \mathbb{R}^{3 \times h_1 \times w_1}$. 
Every other generator $g_i$ takes as an input an upsampled version $\hat{I}^{\uparrow}_{i-1}$ of the previous output $\hat{I}_{i-1}$, and a noise $z_i\in \mathbb{R}^{3 \times h_i \times w_i}$ (whose coordinates are i.i.d. normally distributed) and recursively generates a sample at scale $i$ as follows:
\begin{align}
\label{eq:g}
\hat{I}_1 := g_1(z_1) := \hat{g}_1(z_1)\,,\quad
\hat{I}_i := g_i(\hat{I}_{i-1},z_i) := \hat{I}^{\uparrow}_{i-1} + \hat{g}_i(\hat{I}^{\uparrow}_{i-1} + z_i) \,, \quad i>1
\end{align}
For each scale $i\in [k]$, we denote by $\mathcal{D}_{I,i}$ the distribution of patches $u_{I,i}$ within $I_i$ and by $\mathcal{D}_{\hat{I},i}$ the distribution of patches $u_{\hat{I},i}$ within $\hat{I}_i$ (for $z_1,\dots,z_i \sim \mathcal{N}(0,\mathbb{I})$). The goal of this method is to train each generator $g_i$ to generate samples $\hat{I}_i$, such that, $\mathcal{D}_{\hat{I},i}$  and $\mathcal{D}_{I,i}$ would closely match.

For this task, the generators $g_i$ are progressively optimized to minimize the 1-Wasserstein distance $W(\mathcal{D}_{\hat{I},i},\mathcal{D}_{I,i})$ between the distributions $\mathcal{D}_{\hat{I},i}$ and $\mathcal{D}_{I,i}$. 
The 1-Wasserstein distance between two distributions $\mathcal{D}_1$ and $\mathcal{D}_2$ is defined as follows:
\begin{equation}\label{eq:w}
\begin{aligned}
W(\mathcal{D}_1,\mathcal{D}_2):=&\max_{d:~\|d\|_{L}\leq 1} \left\{\mathop{\mathbb{E}}_{u \sim \mathcal{D}_1} d(u) - \mathop{\mathbb{E}}_{u \sim \mathcal{D}_2} d(u) \right\}, \\
\end{aligned}
\end{equation}
where $\|d\|_{L}$ is the Lipschitz norm of the discriminator $d$.

In general, computing the maximum in Eq.~\ref{eq:w} is intractable. Therefore,~\cite{pmlr-v70-arjovsky17a} suggested to estimate the 1-Wasserstein distance using a pseudo-metric $W_{\mathcal{C}}(\mathcal{D}_1,\mathcal{D}_2)$, where $d$ is parameterized using a neural network from a wide class $\mathcal{C}$.
The method minimizes the adversarial loss, derived from Eq.~\ref{eq:w}, 
\begin{equation}
\mathcal{L}_{adv}(g_i,d_i) :=  \mathop{\mathbb{E}}_{z_{1:i}} [d_i(u_{\hat{I},i})] - \mathop{\mathbb{E}}_{u_{I,i}} [d_i(u_{I,i})],
\end{equation}
where $z_{1:i} = (z_1,\dots,z_{i})$, $u_{\hat{I},i} \sim  \mathcal{D}_{\hat{I},i}$ and $u_{I,i} \sim  \mathcal{D}_{I,i}$.
The above objective is minimized with respect to the parameters of $g_i$, and maximize it with respect to the parameters of the discriminator $d_i$, while freezing the parameters of all previous generators $g_1,\dots,g_{i-1}$. Note that $\hat{I}_i$ is given by $g_i$ according to Eq.~\ref{eq:g}.
In order to guarantee that $d_i$ is of a bounded Lipschitz constant, in~\cite{wgangp} they apply an additional gradient penalty loss to regularize the Lipschitzness of the discriminator:
\begin{equation}
\mathcal{L}_{lip}(d_i) := \mathop{\mathbb{E}}_{u_{I,i}} \left[\|\nabla_u d_i(u_{I,i})\|^2_2\right],    
\end{equation}
In addition, they employ a reconstruction loss. We let $z^0_1$ be a fixed random noise, such that:
\begin{align}
\label{eq:z0}
\hat{I}^0_1 := \hat{g}_1(z^0_1)\,,\quad
\hat{I}^0_i := \hat{I}^{0,\uparrow}_{i-1} + \hat{g}_i(\hat{I}^{0,\uparrow}_{i-1}) \,, \quad i>1
\end{align}
In practice, the expected values with respect to the various distributions are replaced with averages over finite sample sets. For simplicity, throughout the paper we will use expectations to avoid clutter.

\subsection{Hypernetworks}

Formally, a hypernetwork $h(z;f(I;\theta_f))$ is a pair of collaborating neural networks, $f$ and $h$. For an input $I$, network $f$, parameterized by a set $\theta_f$ of trainable parameters, produces the weights $\theta_I = f(I;\theta_f)$ for the {\em primary network} $h$. The network $h$ takes an input $z$, and returns an output $h(z;\theta_I)$ that depends on both $z$ and the task specific input $I$. In practice, $f$ is typically a large neural network and $h$ is a small neural network.  Throughout the paper, we use ``$;$'' as a separator between the input and trainable parameters of a neural network.

\section{Method}
\label{sec:method}
Our method solves an inherent limitation of current single-sample GAN models, which is the scaling to multi-sample learning, such that the same network can perform single-image generations for each sample. For this purpose, we adopt a hypernetwork based modeling for the involved generators and discriminators. 
The hypernetwork network $f_g$ produces weights for each $g_i$, and a hypernetwork $f_d$  produces weights for each $d_i$. In this setting, $g_i$ and $d_i$ consist of the same architecture presented in Sec.~\ref{sec:back}, and serve as the primary networks for $f_g$ and $f_d$ (resp.).

Two alternatives for the proposed setting are presented for completeness and are briefly discussed in Sec.~\ref{sec:analysis}, (i) shared discriminator and (ii) shared feature extractor. These alternatives help in understanding our proposed approach. The full description and proofs are presented in the appendix.
An illustration of the proposed model and the two variants are presented in Fig.~\ref{fig:our_arch}.

\begin{figure}[t]
	\centering
	\begin{tabular}{ccc}
	\includegraphics[height=93px]{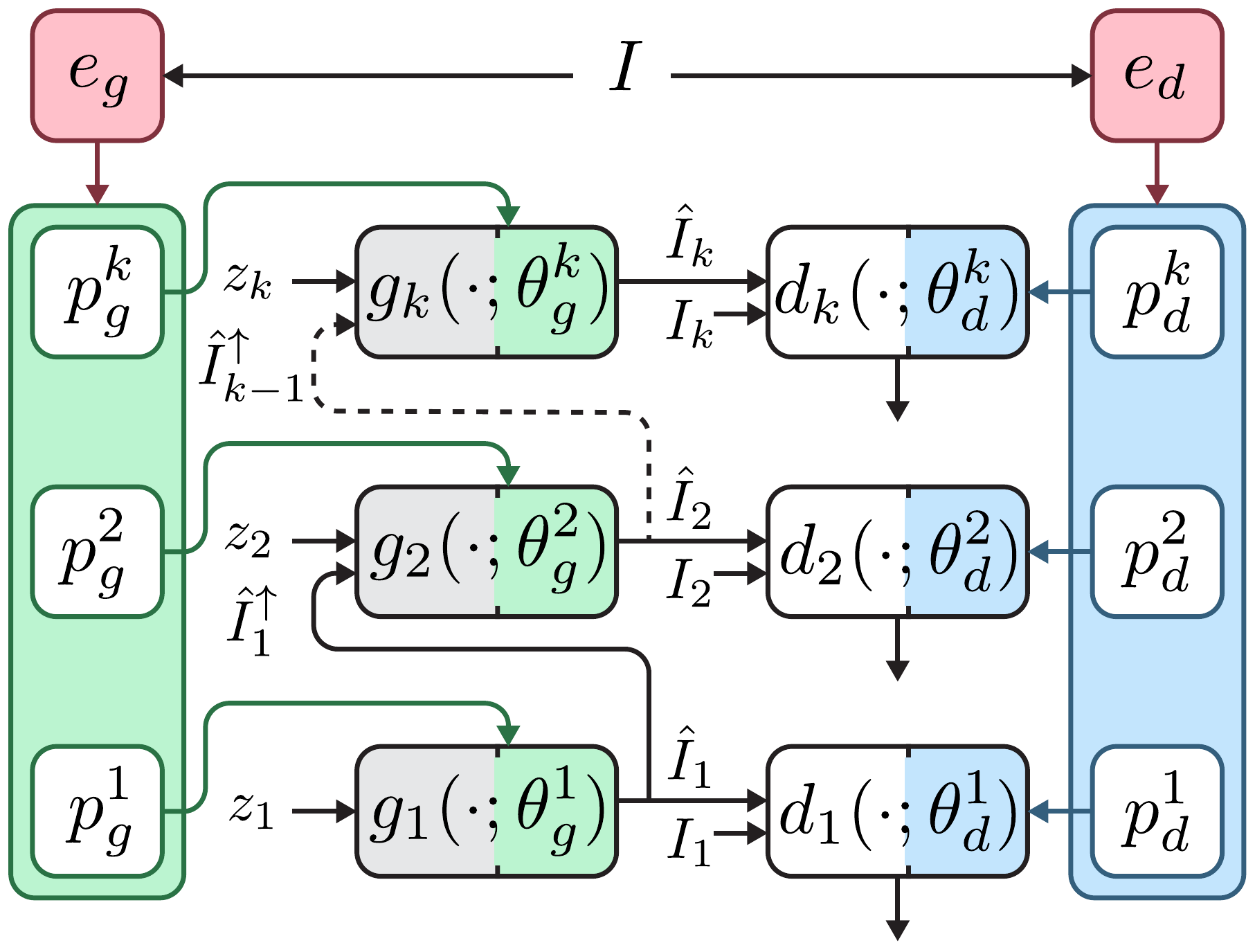}&
	\includegraphics[height=106px]{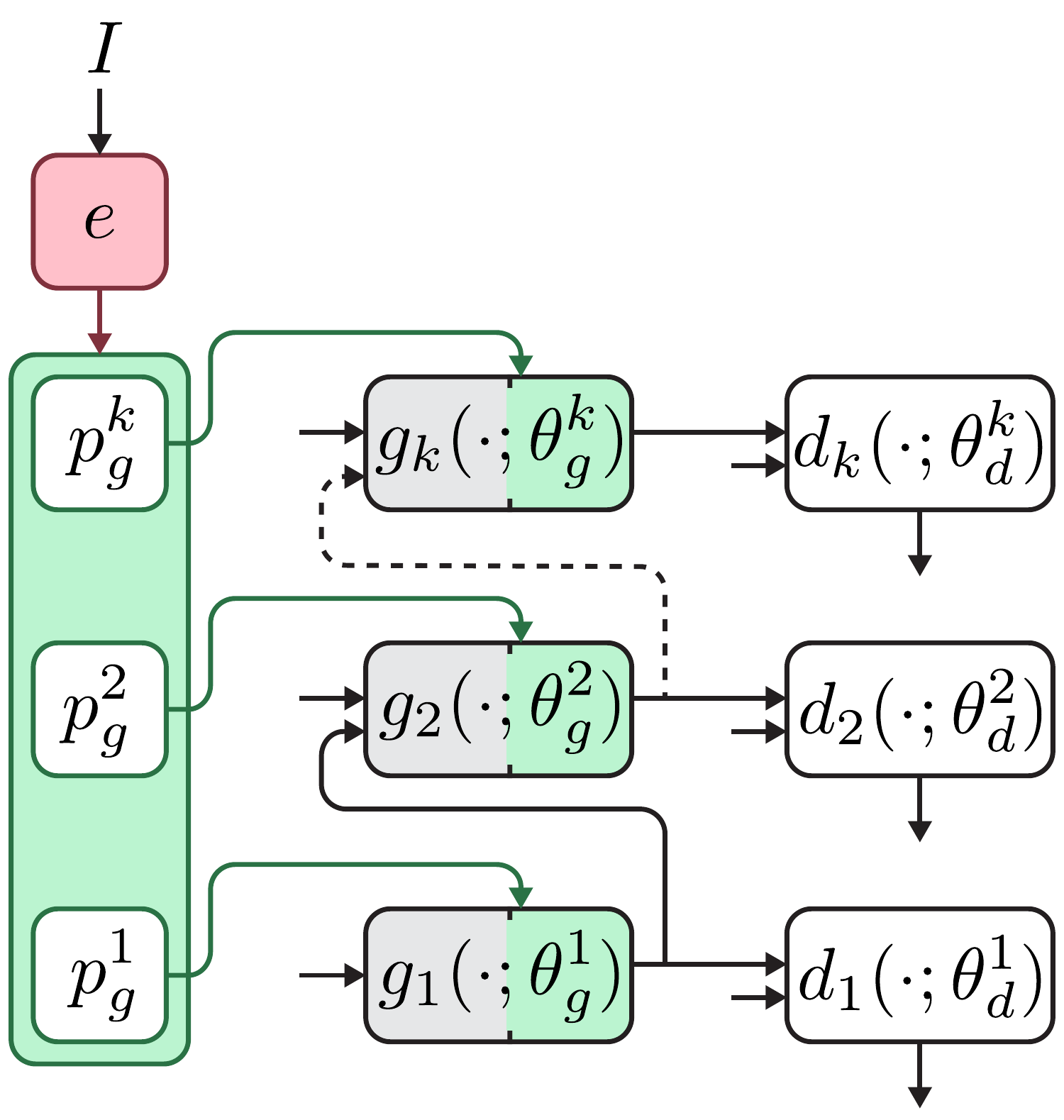}&
    \includegraphics[height=106px]{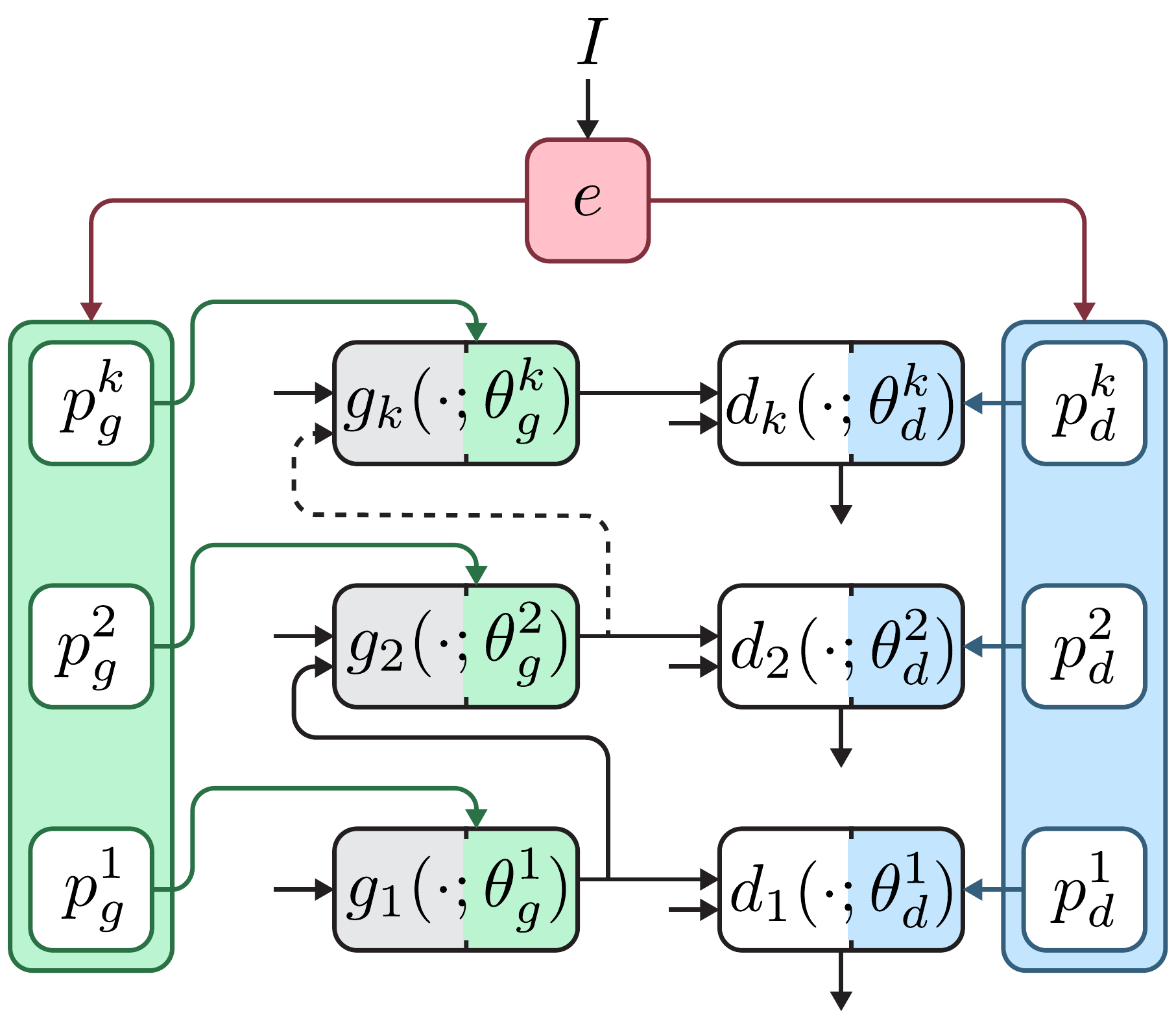}\\
    (a) & (b) & (c)
    \end{tabular}
    \caption{{\bf Alternative architectures for hypernetwork single image generators.} {\bf(a)} Our model architecture, consists of two embedding networks - $e_g$ for the {\em hyper-generator} and $e_d$ for the {\em hyper-discriminator}. The primary networks $g_i$ and $d_i$ follows Sec.~\ref{sec:single_image_generation}. {\bf (b)} Hyper-Generator with shared discriminator. {\bf (c)} Shared feature extractor. We omit the input/output names for clarity.}
    \label{fig:our_arch}
    \label{fig:shared_arch}
\end{figure}

\subsection{The Model}
Our model consists of two main components: a {\em hyper-generator} and a {\em hyper-discriminator} (see Fig~\ref{fig:our_arch}(a) for an illustration). The hyper-generator $g_i$ is a hypernetwork that is defined as follows:
\begin{equation}
\begin{aligned}
g_i(z,I) &:= g_i(z;f^i_g(I;\theta_{f_g})),
\end{aligned}
\end{equation}
where $f^i_g(I;\theta_{f_g})$ is a neural network that takes an input image $I$, and returns a vector of weights for the $i$'th generator $g_i$. This network is decomposed into an embedding network $e_g$ that is shared among scales and a linear projection $p^i_g$ per scale,
\begin{align}
    E_g(I) &:= e_g(I;\theta_{e_g})\\
    f^i_g(I;\theta_{f_g}) &:= p^i_g(E_g(I);\theta^i_{p_g}) 
\end{align}

The network $e_g$ is parameterized with a set of parameters $\theta_{e_g}$, and each projection $p^i_g$ is parameterized with a set of parameters $\theta^i_{p_g}$ (for simplicity, we denote their union by $\theta_{p_g} = (\theta^i_{p_g})^{k}_{i=1}$). {Each $g_i$ is a fully-convolutional network, following Sec.~\ref{sec:single_image_generation}, whose weights are $\theta_g^i := f^i_g(I;\theta_{f_g})$.} The overall set of trainable parameters within $g_i$ is $\theta_{f_g} := (\theta_{e_g},\theta^i_{p_g})^{k}_{i=1}$. 
The hyper-discriminator is defined in a similar manner:
\begin{equation}
d_i(u,I) := d_i(u;f^i_d(I;\theta_{f_d}))
\end{equation}
where $f^i_d(I;\theta_{\theta_d})$ is a network that takes an image $I$ and returns a vector of weights for the $i$'th discriminator $d_i$. This network is also decomposed into an embedding network and a set of projections:
\begin{align}
    E_d(I) &:= e_d(I;\theta_{e_d})\\
    f^i_d(I;\theta_{f_d}) &:= p_d(E_g(I);\theta^i_{p_d})
\end{align}

In contrast to the generator, the hyper-discriminator works only on the last image scale.
Each $d_i$ is a fully-convolutional network, following Sec.~\ref{sec:single_image_generation}, whose weights are $\theta^i_d := f^i_d(I;\theta_{f_d})$. The overall set of trainable parameters within $d_i$ is $\theta_{f_d} := (\theta_{e_d},\theta^i_{p_d})^{k}_{i=1}$.

\subsection{Loss Functions} 

Our objective function is decomposed into an adversarial and reconstruction loss functions,
\begin{align}
\mathcal{L}(g_i,d_i) = &\mathcal{L}_{adv}(g_i,d_i) + \lambda_1 \cdot \mathcal{L}_{lip}(d_i) + \lambda_2 \cdot \mathcal{L}_{acc-rec}(g_i),
\end{align}
where $\lambda_1,\lambda_2 > 0$ are two tradeoff parameters. The loss functions are described below.

\noindent\textbf{Adversarial Loss Function \quad} Our adversarial loss function is defined in the following manner:
\begin{equation}
\begin{aligned}
\mathcal{L}_{adv}(g_i,d_i) := \mathbb{E}_I \left\{\mathop{\mathbb{E}}_{z_{1:i}} d_i(u_{\hat{I},i};f^i_d(I)) - \mathop{\mathbb{E}}_{u_{I,i}} d_i(u_{I,i} ;f^i_d(I)) \right\},
\end{aligned}
\end{equation}
which is maximized by $\theta_{f_d}$ and minimized by $\theta_{f_g}$. In order to suffice that $d_i$ would have a bounded Lipschitz constant, we apply the gradient penalty loss function:
\begin{equation}
\mathcal{L}_{lip}(d_i) := \mathbb{E}_{I} \mathbb{E}_{u_{I,i}} \left[\|\nabla_u d_i(u_{I,i}) \|^2_2 \right] 
\end{equation}
For a theoretical analysis of the sufficiency of these loss functions, see Sec.~\ref{sec:analysis-ours}.

\noindent\textbf{Reconstruction Loss Function \quad}
Our method also employs a similar loss function to the reconstruction loss defined in Sec.~\ref{sec:single_image_generation}. We accumulate all previous reconstruction losses for each scale: 
\begin{align}
    \mathcal{L}_{acc-rec}(g_i) := \mathbb{E}_i \mathcal{L}_{rec}(\hat{I}_i, I_i)
\end{align}
Previous methods~\cite{shaham2019singan,hinz2021improved,gur2020hierarchical} freeze each intermediate generator $g_i$ except for the current training scale, ensuring each $g_i$ to be independent
\footnote{In~\cite{hinz2021improved} they optimized each $g_i$ with its $j$ (constant) neighboring scales.}. 
In our case, we freeze the projection of all previous scales, except the current scale. However, because $e_g$ is shared for all projections, the accumulated reconstruction loss regularizes the training of $e_g$, by minimizing the reconstruction loss with freezed projections as well. We note that this accumulation is mostly needed for small datasets, where for large ones we simply compute the loss with respect to the last scale.

\subsection{Initialization and Optimization} 
We initialize the hypernetworks with the initialization suggested by~\cite{littwin2020optimization}. In this initialization, the network $f$ is initialized using the standard Kaiming He initialization~\cite{10.1109/ICCV.2015.123}. Each convolutional layer in the primary networks $g_i$ and $d_i$ has a $\frac{1}{\sqrt{c_{in} \cdot K \cdot K}}$ normalization, where $c_{in}$ is the number of input channels, $K \times K$ is the kernel size of the convolution layer.

{We progressively train the model, starting from scale $1$ onward to scale $k$. As noted, during training we freeze all previous projection layers, except for the current training scale. The networks $f_g$ and $f_d$ are continuously trained across scales, where for $f_g$ we add additional projection layers for each new scale, initialized by the previous scale, while $f_d$ does not change. Each scale is trained for a constant number of iterations, and optimized using the Adam~\cite{kingma2014adam} optimizer. Full training and experiments settings are presented in the appendix.}

\section{Theoretical Analysis}

\label{sec:analysis}
In this section, we analyze the soundness of our method, showing the sufficiency of our method.
In the appendix we show the importance of the hyper-discriminator and that the generator and discriminator should inherit their parameters from two disjoint hypernetworks. For simplicity, throughout the analysis we omit the reconstruction loss (i.e., $\lambda_2=0$), and assume that the distributions $\mathcal{D}_I$, $\mathcal{D}_{\hat{I},i}$ and $\mathcal{D}_{I,i}$ are supported by bounded sets.
Proof for each proposition are available in the appendix.

\subsection{Our Architecture}\label{sec:analysis-ours}

In general, we are interested in finding a hyper-generator $g_i$ for each scale, such that, for each image $I$, $g_i(\cdot,I)$ would generate samples $\hat{I}_i$ whose patches $u_{\hat{I},i}\sim \mathcal{D}_{\hat{I},i}$ are similar to the patches $u_{I,i}\sim \mathcal{D}_{I,i}$ within $I_i$. Specifically, we would like to train the parameters of $g_i$  to minimize the following function:
\begin{equation}
\begin{aligned}
&\mathbb{E}_I W_{\mathcal{C}}(\mathcal{D}_{\hat{I},i},\mathcal{D}_{I,i}) = \mathbb{E}_I \max_{d^I_i \in \mathcal{C}^1} \big\{\mathop{\mathbb{E}}_{z_{1:i}} d^I_i(u_{\hat{I},i}) - \mathop{\mathbb{E}}_{u_{I,i}} d^I_i(u_{I,i}) \big\},
\end{aligned}
\end{equation}
where $\mathcal{C}^{\alpha} := \mathcal{C} \cap \{d_i \mid \|d_i\|_{L} \leq \alpha\}$. However, to directly minimize this objective function, one needs to be able to hold a different discriminator $d^I_i$ for each sample $I$, which is computationally expensive.

Fortunately, we can think about this expression in a different manner, as the above expression can also be written as follows:
\begin{equation}\label{eq:S}
\begin{aligned}
&\mathbb{E}_I W_{\mathcal{C}}(\mathcal{D}_{\hat{I},i},\mathcal{D}_{I,i}) = \max_{S} \mathbb{E}_I\big\{\mathop{\mathbb{E}}_{z_{1:i}} d_i(u_{\hat{I},i};S(I)) - \mathop{\mathbb{E}}_{u_{I,i}} d_i(u_{I,i};S(I)) \big\},
\end{aligned}
\end{equation}
where the maximum is taken over the set of mappings $S$ from images $I$ to parameters $\theta_I$ of discriminators $d^I_i \in \mathcal{C}^1$. We let $S^*$ be a mapping that takes $I$ and returns the parameters $S^*(I)$ of the discriminator $d^I_i := d_i(\cdot;S^*(I)) = \arg\max_{d_i \in \mathcal{C}^1} \left\{\mathbb{E}_{z_{1:i}} d^I_i(u_{\hat{I},i}) - \mathbb{E}_{u_{I,i}} d^I_i(u_{I,i}) \right\}$. 

Therefore, if $S^*$ can be approximated by a large neural network $f^i_d(I) = f^i_d(I;\theta_{f_d}) \approx S^*(I)$, then, we can approximately solve the maximization in Eq.~\ref{eq:S} by parameterizing the discriminator with a hypernetwork $d_i := d_i(u;f_d(I;\theta_{f_d}))$ and training its parameters to (approximately) solve the maximization in Eq.~\ref{eq:S}. For instance if $S^*$ is a continuous function, one can approximate $S^*$ using a large enough neural network up to any approximation error $\leq \epsilon$~\cite{Cybenko1989,Hornik1991ApproximationCO,Mhaskar:1996:NNO:1362203.1362213,NIPS2017_7203,hanin2017approximating,10.5555/3327345.3327515,pmlr-v70-safran17a}. This is summarized in the following proposition. 

\begin{proposition} 
\label{prop:1}
Assume that $\mathcal{I} \subset \mathbb{R}^{3 \times h \times w}$ is compact. Let $\epsilon>0$ be an approximation error. Let $g_i(z,I) := g_i(z;f^i_g(I;\theta_{f_g}))$ be a hyper-generator and $\mathcal{C}$ a class of discriminators. Assume that $S^*$ is continuous over $\mathcal{I}$. Then, there is a large enough neural network $f^i_d$ (whose size depends on $\epsilon$), such that, the hyper-discriminator $d_i(u,I) := d_i(u;f^i_d(I;\theta_{f_d}))$ satisfies:
\begin{equation}
\begin{aligned}
&\mathbb{E}_I W_{\mathcal{C}}(\mathcal{D}_{\hat{I},i},\mathcal{D}_{I,i}) = \max_{\theta_{f_d}} \mathcal{L}_{adv}(g_i,d_i) + o_{\epsilon}(1),
\end{aligned}
\end{equation}
where the maximum is taken over the parameterizations $\theta_{f_d}$ of $f_{d}$, such that, $d_i(\cdot; f^i_d(I;\theta_{f_d})) \in \mathcal{C}^1$.
\end{proposition}
A proof for the existence of a continuous selector $S^*(I)$ has been proposed~\cite{galanti2020modularity,10.1006/jath.1998.3305,Maiorov99lowerbounds} for similar settings, and the proof for Prop.~\ref{prop:1} is provided as part of the supplementary material.
According to this proposition, in order to minimize $\mathbb{E}_I W_{\mathcal{C}}(\mathcal{D}_{\hat{I},i},\mathcal{D}_{I,i})$, we can simply parameterize our discriminator with a hypernetwork $d_i := d_i(u;f^i_d(I;\theta_{f_d}))$ and to train the hyper-generator $g_i$ to solve: $\min_{\theta_{f_g}} \max_{\theta_{f_d}} \mathcal{L}_{adv}(d_i,g_i)$ along with the gradient penalty loss $\mathcal{L}_{lip}(d_i)$ to ensure that $d_i(\cdot;f_d(I;\theta_{f_d}))$ would have a bounded Lipschitz constant. 

Differently said, in order to guarantee that the approximation error in Prop.~\ref{prop:1} would be small, we advocate {\bf selecting the hypernetwork $f^i_d$ to be a large neural network}. In this case, if we are able to effectively optimize $\theta_{f_g}$ and $\theta_{f_d}$ to solve $\min_{\theta_{f_g}} \max_{\theta_{f_d}} \mathcal{L}_{adv}(d_i,g_i)$ (s.t the Lipschitz constant of $d_i$ is bounded), we can ensure that $\mathbb{E}_I W_{\mathcal{C}}(\mathcal{D}_{\hat{I},i},\mathcal{D}_{I,i})$ would be small as desired.

\subsection{Alternative Architectures}
\label{sec:alt_arch}
As presented in Sec.~\ref{sec:method}, we consider two alternative architectures (i) shared discriminator and (ii) shared feature extractor. We briefly describe each proposed variant and its limitations, the full analysis is presented in the supplementary material.

\noindent\textbf{Shared Discriminator}\quad
In this case, the model has two main components for each scale $i$: a hyper-generator $g_i(z,I) = g_i(z;f^i_g(I;\theta_{f_g}))$ along with a standard discriminator $d_i(u) = d_i(u;\theta_d)$ that is shared among all samples $I$, as illustrated in Fig.~\ref{fig:our_arch}(b).
We show that if the expected (w.r.t $I \sim \mathcal{D}_I$) distance between the distributions $\mathcal{D}_{\hat{I},i}$ and $\mathcal{D}_{i}$ is small, then, the loss function $\mathcal{L}_{adv}(g_i,d_i) := \mathbb{E}_{I} \{\mathop{\mathbb{E}}_{z_{1:i}} d_i(u_{\hat{I},i}) - \mathop{\mathbb{E}}_{u_{i,I}} d_i(u_{I,i}) \}$ tends to be small. Here, $\mathcal{D}_{i}$ denotes the distribution of $u_{\hat{I},i}\sim \mathcal{D}_{\hat{I},i}$ for $I \sim \mathcal{D}_I$. This proposition shows that a hyper-generator $g_i(\cdot,I)$ that generates samples $\hat{I}_i$ whose patches are similar to samples of $ \mathcal{D}_{i}$ would minimize the loss function $\mathcal{L}_{adv}(g_i,d_i)$, even though the generated samples are not conditioned on the image $I$. 
Therefore, solely minimizing the adversarial loss would not guarantee that $g_i(\cdot,I)$ would generate samples $\hat{I}_i$ that are similar to $I_i$.

\noindent\textbf{Shared Feature Extractor}\quad
We note that as a strategy for reducing the number of trainable parameters in the whole model, one could restrict $f_g$ and $f_d$ to share their encoding component $e$, as illustrated in Fig.~\ref{fig:our_arch}(c). We show two failure cases of this approach. 
First, we consider the case where the model is trained using GD. In this case, GD iteratively updates $(\theta_e,\theta^i_{p_g})$ to minimize $\mathcal{L}_{adv}(g_i,d_i)$ and updates $(\theta_e,\theta^i_{p_d})$ to maximize $\mathcal{L}_{adv}(g_i,d_i) - \lambda_1 \cdot \mathcal{L}_{lip}(d_i)$. 
Informally, we show that $\theta_e$ is essentially trained to only minimize $\mathcal{L}_{lip}(d_i)$ and that each tuple $(\theta_e,\theta^i_{p_g},\theta^i_{p_d})$ with $d_i \equiv 0$ is an equilibrium point. 
In addition, we note that $\mathcal{L}_{lip}(d_i)$ is minimized by $d_i\equiv 0$. Therefore, it is likely that $d_i$ would converge to $0$ during training. Meaning, that at some point the discriminator is ineffective. In particular, if $\theta_e=0$, then, $(\theta_e,\theta^i_{p_g},\theta^i_{p_d})$ is an equilibrium point. We note that $\theta_e=0$ is not a desirable output of the training algorithm, since it provides a hyper-generator $g_i(\cdot,I)$ that is independent of the input image $I$.
Second, we consider the case where GD iteratively optimizes $(\theta_e,\theta^i_{p_g})$ to minimize $\mathcal{L}_{adv}(g_i,d_i)$, $\theta^i_{p_d}$ to maximize $\mathcal{L}_{adv}(g_i,d_i)$ and $(\theta_e,\theta^i_{p_g})$ to minimize the loss $\mathcal{L}_{lip}(d_i)$. We show that each tuple $(\theta_e,\theta^i_{p_g},\theta^i_{p_d})$ with $\theta_e=0$ is again, an equilibrium point.

\section{Experiments}

{Our experiments are divided into two parts. In the first part, we study three different training regimes of our method. First, we experiment with single-image training in order to produce a fair comparison to preexisting methods. Second, we present a mini-batch training scheme, where instead of a single image, the model is trained on a fixed set of images. Lastly, we experiment with training over a full dataset, that cannot fit into a single batch. 

In the second part, we experiment with several applications of our method. Specifically, we study the ability of our method in the Harmonization, Editing and Animation tasks proposed by~\cite{shaham2019singan}, {as well as generating samples of arbitrary size and aspect ratio}. In addition, we also experiment with two new applications: image interpolations, and generation at inference time. These application are unique to multi-image training.
 
Due to space constraints, we focus on our novel applications, and refer the reader to the appendix for our full set of applications, as well as technical details, such as, specific hyperparameters, GPU usage and additional experiments.}

{Throughout the experiments, we consider the following set of baselines}:  SinGAN~\cite{shaham2019singan}, ConSinGAN~\cite{hinz2021improved} and HP-VAE-GAN~\cite{gur2020hierarchical}. 
{\color{black}To evaluate image generation, we use the single-image FID metric (SIFID)~\cite{shaham2019singan}. Following~\cite{shaham2019singan}, the metric represents the mean of minimum SIFID over 100 generated samples per image. We further compute the mean-SIFID (mSIFID), which is the mean across all generated samples of all images, without taking the minimum for each image.}

As simply overfitting the training image would lead to a SIFID value of 0, a diversity measure is necessary. 
{For this purpose, we employ the diversity measure used in~\cite{shaham2019singan}.} This measure is computed as the averaged standard deviation over all pixel values along the channel axis of 150 generated images.

 Previous works in the field~\cite{shocher2018ingan,shaham2019singan} require training on each image independently. In order to compare with previous work, we use the 50-image dataset of~\cite{shaham2019singan}, denoted by Places-50 and the 50-image dataset of ~\cite{consingan}, denoted by LSUN-50. Additionally, whenever a quantitative measure is available, we present competitive results, and qualitatively, our results are at least as good, if not better than those of the single-image GANs. For larger datasets, that consist of up to 5000 images, we perform thorough experiments with our proposed method. 
The dataset presented by SinGAN, Places-50, consists of 50 images randomly picked from subcategories of the Places dataset~\cite{NIPS2014_3fe94a00} -- Mountains, Hills, Desert and Sky and the dataset presented by ConSinGAN, LSUN-50, consists of five randomly sampled images from each of the ten LSUN dataset categories. In order to evaluate our method on larger datasets, we consider three subsets of the {\em Valleys} category of the Places dataset; the first 500(V500), 2500(V2500) and 5000(V5000) (the entire category) images, and use the 100 images test-set when relevant. Additionally, we consider the first 250(C250) and 500(C500) images of the Churches Outdoor category of the LSUN dataset.

\begin{table}[t]
    \centering
    \caption{\textbf{Quantitative comparison on Places-50/LSUN-50}, showing SIFID, mSIFID, diversity and training time per image (minutes). Our method shows comparable results to single-image models in both {\em single} and {\em dataset} settings, where the overall training time per image is significantly lower.}
    \begin{tabular*}{\linewidth}{@{\extracolsep{\fill}}l@{~~}ccc@{~}c}
    \toprule
    Method & SIFID $\downarrow$ & mSIFID $\downarrow$ & Diversity$\uparrow$ & min./image$\downarrow$ \\
    \midrule
    SinGAN~\cite{shaham2019singan} & 0.09/0.11 & 0.15/0.20 & 0.52/0.60 &60\\
    ConSinGAN~\cite{hinz2021improved} & 0.06/0.08 & 0.08/0.13 & 0.50/0.55& 20\\
    HP-VAE-GAN~\cite{gur2020hierarchical} & 0.17/0.40 & 0.27/0.62 & 0.62/0.78 &60\\
    Ours {\em Single} & 0.03/0.11 & 0.06/0.19 & 0.57/0.65 & 30\\
    \midrule
    Ours {\em Dataset} & 0.05/0.11 & 0.07/0.16 & 0.50/0.48 & 5\\
    \bottomrule
    \end{tabular*}
    \label{tab:singan_test_set}
\end{table}

\begin{table}[t]
\begin{minipage}[c]{0.49\linewidth}
    \vspace{-15px}
    \captionof{table}{{\bf Varying the batch size in single mini-batch training.} Both SIFID and diversity (w.r.t a specific batch size) remain stable regardless of the size of the mini-batch.}
    \begin{tabular*}{\linewidth}{@{\extracolsep{\fill}}lccc}
    \toprule
    Batch Size & SIFID$\downarrow$ & mSIFID$\downarrow$ & Diversity$\uparrow$ \\
    \midrule
    1 & 0.03 & 0.07 & 0.73 \\
    2 & 0.04 & 0.07 & 0.66 \\
    3 & 0.03 & 0.07 & 0.68 \\
    4 & 0.04 & 0.08 & 0.70 \\
    5 & 0.04 & 0.08 & 0.71 \\
    \bottomrule
    \end{tabular*}
    \label{tab:batch_size}
\end{minipage}%
\hfill
\begin{minipage}[c]{0.49\linewidth}
    \setlength{\tabcolsep}{1pt} 
    \renewcommand{\arraystretch}{1} 
    \begin{tabular}{ccc}
    \includegraphics[width=0.32\linewidth]{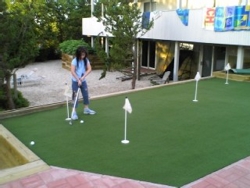}&
    \includegraphics[width=0.32\linewidth]{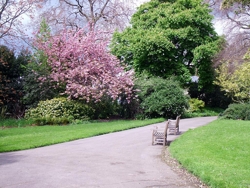}&
    \includegraphics[width=0.32\linewidth]{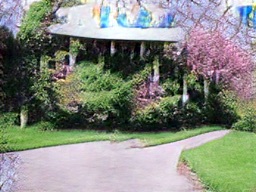}\\
    (a) & (b) & (c)
    \end{tabular}
    \captionof{figure}{{\bf Leakage in the multi-image training when using a shared discriminator.} {\bf(a)} and {\bf(b)} are the two training images, and {\bf(c)} is a generated image for the model of image {\bf(b)}. As can be seen, it contains patches from the first image as well.}
    \label{fig:leakage}
\end{minipage}
    \centering
    \caption{{\color{black}\textbf{Performance on the Valley dataset}, showing SIFID, mSIFID, diversity and train time per image (minutes). As can be seen, inference performance increases with training data size.}}
    \begin{tabular*}{\linewidth}{@{\extracolsep{\fill}}lccccccc}
    \toprule
    \multirow{2}{*}{Dataset} & \multicolumn{4}{c}{Train} & \multicolumn{3}{c}{Test}\\
    \cmidrule(lr){2-5}
    \cmidrule(lr){6-8}
     & SIFID$\downarrow$ & mSIFID$\downarrow$& Diversity$\uparrow$  & min./image$\downarrow$ & SIFID$\downarrow$ & mSIFID$\downarrow$ & Diversity$\uparrow$\\
    \midrule
     Valley$_{500}$ & 0.04 &0.07 & 0.51 &  4.0& 0.47&2.47 & 0.34\\
    Valley$_{2500}$ & 0.05 & 0.08  & 0.52 & 3.5 & 0.43  & 1.86  & 0.37  \\
    Valley$_{5000}$ & 0.05 & 0.08  & 0.51 &  3.0 & 0.41 & 1.52  & 0.40 \\
    \bottomrule
    \end{tabular*}
    \medskip
    \label{tab:datasets}
    \vspace{-10px}
\end{table}

\subsection{Training Procedures}

\noindent\textbf{Single-Image training \quad}
Our approach is first evaluated when training with a single image, as done in previous methods.
Since a single function needs to be learned, a standard discriminator (\ie, not learned via hypernetwork) is used in this specific case in order to avoid redundant enlargement of the model and speed up training. Similar results are obtained using a hyper-discriminator. 
Tab.~\ref{tab:singan_test_set} shows that our performance is on par with current single-image models on this setting. 

\noindent\textbf{Single mini-batch training \quad}
When introduced with multiples images, the standard discriminator, as for the baseline methods, suffers from leakage between the images in the mini-batch, \ie, the patches of the generated images are distributed as patches of arbitrary images from the batch (Sec.~\ref{sec:alt_arch}-- Shared Discriminator). Fig.~\ref{fig:leakage} illustrates this effect. To overcome this issue, we introduce a hyper-discriminator which allows efficiently to learn a different discriminator model per image. 
To evaluate performance on single mini-batch learning, we randomly sampled a set of 5 images from the 50 images dataset and trained a different model for each permutation of the set of size $1 \leq i \leq 5$. Tab.~\ref{tab:batch_size} show performance is good regardless of the mini-batch size, which indicates the hypernetwork model successfully learns a different model for each input image.

\begin{figure}
	\centering
	\includegraphics[width=\linewidth]{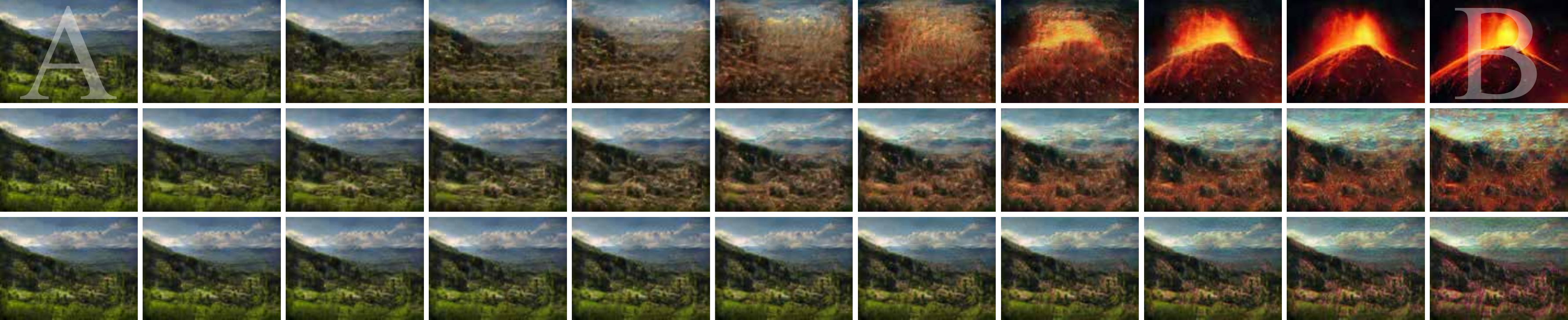}
	\captionof{figure}{\textbf{Interpolation in the space of generative networks}. A hypernetwork is trained to produce unique Single-Image Generators from a dataset of 50 images. \textbf{Top left (right)} - a generated image from generator A (B). Each column represents different mixtures of the generators' latent representations. Each row represents injection of the mixed representation at different scales, where all previous scales use generator A representation 
	- from coarsest {\bf (top)} to finest {\bf (bottom)}.}
	\label{fig:interpolation}
	\medskip

	\includegraphics[scale=0.3]{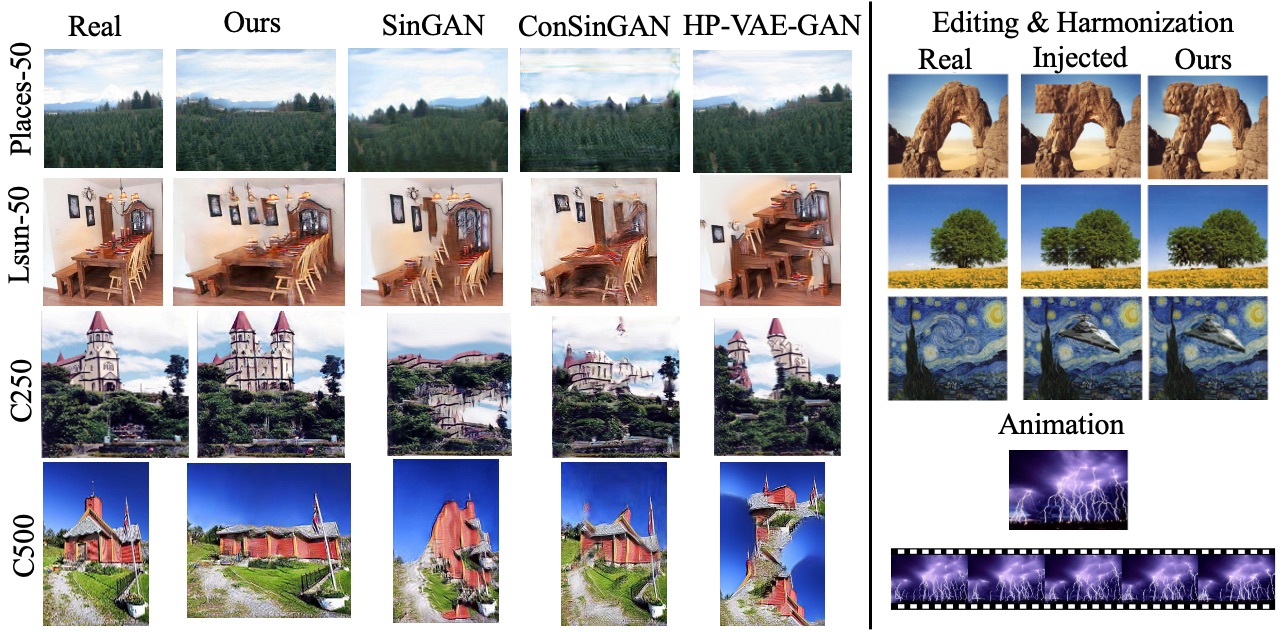}\captionof{figure}{{\color{black}\textbf{Left:} Comparison of image generation results with single image baselines on different datasets. \textbf{Right:} Results of applications, trained with the Places-50 dataset. Our method allows us to manipulate images such as Editing, Harmonization and Animation at a large scale, training all images at once.}}
    \label{fig:mainwbaselines}

    \setlength{\tabcolsep}{0.5pt} 
    \renewcommand{\arraystretch}{0.5} 
    \centering
    \small
    \begin{tabular}{ccccccc@{~~}cccc}
    \multicolumn{7}{c}{Train} & \multicolumn{4}{c}{Test}\\
    \cmidrule(lr){1-7}
    \cmidrule(lr){8-11}
    Real & \cite{shaham2019singan} & \cite{consingan} & \cite{gur2020hierarchical} & V500 & V2500 & V5000 & Real & V500 & V2500 & V5000\\
    \raisebox{11mm}{\multirow{1}{*}{\includegraphics[width=0.0875\linewidth]{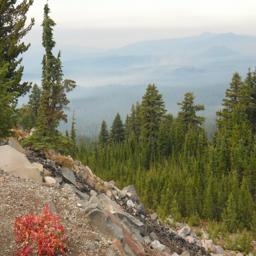}}}&
    \includegraphics[width=0.0875\linewidth]{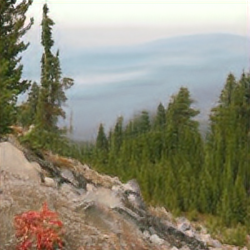}&
    \includegraphics[width=0.0875\linewidth]{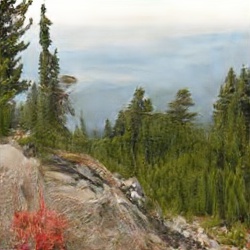}&
    \includegraphics[width=0.0875\linewidth]{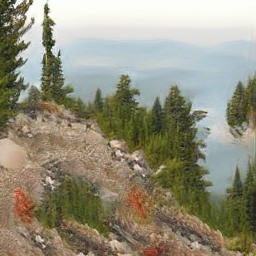}&
    \includegraphics[width=0.0875\linewidth]{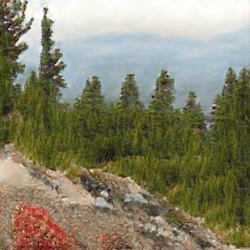}&
    \includegraphics[width=0.0875\linewidth]{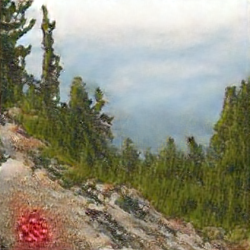}&
    \includegraphics[width=0.0875\linewidth]{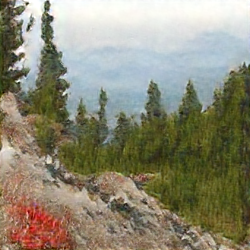}&
    \raisebox{11mm}{\multirow{1}{*}{\includegraphics[width=0.0875\linewidth]{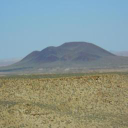}}}&

    \includegraphics[width=0.0875\linewidth]{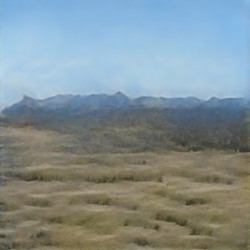}&
    \includegraphics[width=0.0875\linewidth]{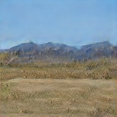}&
    \includegraphics[width=0.0875\linewidth]{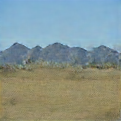}\\
    \end{tabular}
    	\captionof{figure}{{\color{black}\textbf{Training and testing results on Valley dataset}. Real images from the train/test-set respectively. Training results include SinGAN~\cite{shaham2019singan}, ConSinGAN~\cite{consingan} and HP-VAE-GAN~\cite{gur2020hierarchical}.}}
    \label{fig:inference}
    \vspace{-5px}
\end{figure}

\noindent\textbf{Dataset training\quad} 
Our main contribution arises from training with large amount of data. Beside Places-50 and LSUN-50, 
we trained our method on three subset of the Valley category, as presented above -- Valley$_{500}$, Valley$_{2500}$ and Valley$_{5000}$, iterating on batches of 16 images, for 25k, 100k, 150k iterations per scale, respectively, on a single GPU. Tab.~\ref{tab:singan_test_set} and~\ref{tab:datasets} shows performance and training time per image. Churches$_{250}$ and Churches$_{500}$ were trained in a similar manner for 20k and 30k iterations per scale respectively and reached equal performance of 0.20, 0.27 and 0.47 for SIFID, mSIFID and diversity metrics.

As far as we are aware, our method is the first to be able to train multiple single image models at this scale with a descent runtime. For example, training the model presented by \cite{shaham2019singan} on Valley$_{5000}$ with a single GPU would require the training of 5000 different and independent models, and would take approximately 200 days. Our method takes 10 days on a single GPU, and thus is faster by a factor of 20.

\subsection{Applications}
{As noted above, we present our novel applications in the main text, and refer the reader to the appendix for applications presented by previous work. }

\noindent\textbf{Interpolation\quad}
As our meta learning approach learns the space of generators, and is trained on multiple images, our model is able to interpolate between different images smoothly and at different scales.
In difference from common interpolation, the hierarchical structure of the model enables the interpolation in different scales as follows: We start by interpolating over the latent representation $e_g$, resulting in a new {\em generator}. Let $A$ and $B$ be two different images, we compute their latent representations $e^A_g = E_g(A)$ and $e^B_g = E_g(B)$ (resp.) and perform linear interpolation between the two, for $\alpha\in[0,1]$, $e^\alpha_g = \alpha e^A_g + (1 - \alpha)e^B_g$ resulting in a new generator.
We then select a primary image, $A$ for example, and initial scale $m$, and perform the following steps: (i) we use $e^A_g$ for the generation of scales $1$ to $m$, and (ii) from scale $m$ onward, we switch to $e^\alpha_g$, and continue the generation accordingly. The result is a mixing at different patch scales, where scale $1$ controls the structure of the image, and the last scale controls the finer texture of the image.
Fig.~\ref{fig:interpolation} shows an example of a pair of images and its interpolation, where the primary image is denoted by $A$ (top-left), and the target image by $B$ (top-right).
We show interpolations at three different scales - first (1), middle and last, presented from top to bottom. As can be seen, interpolating on the first scale results in more structural changes, while interpolating on middle and last scales results in a more textural changes. By changing $\alpha$ we are able to obtain a wide gamut of intermediate 
options between the two images.

\begin{figure}[t]
    \setlength{\tabcolsep}{0.5pt} 
    \renewcommand{\arraystretch}{0.5} 
    \centering
    \begin{tabular}{lcccc@{~~}ccccccc}
    & \multicolumn{4}{c}{Valley$_{5000}$} & MRI & Retina & Spec. & \multicolumn{2}{c}{Monet} & \multicolumn{2}{c}{Van Gogh}\\
    \raisebox{5mm}{\rotatebox[origin=c]{90}{\parbox{0.5cm}{\scriptsize \centering Input}}}&
    \includegraphics[width=0.085\linewidth]{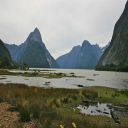}&
    \includegraphics[width=0.085\linewidth]{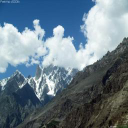}&
    \includegraphics[width=0.085\linewidth]{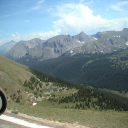}&
    \includegraphics[width=0.085\linewidth]{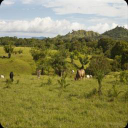}&
    \includegraphics[width=0.085\linewidth]{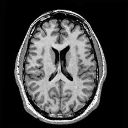}&
    \includegraphics[width=0.085\linewidth]{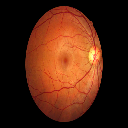}&
    \includegraphics[width=0.085\linewidth]{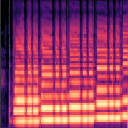}&
    \includegraphics[width=0.085\linewidth]{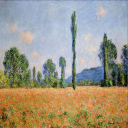}&
    \includegraphics[width=0.085\linewidth]{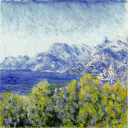}&
    \includegraphics[width=0.085\linewidth]{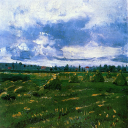}&
    \includegraphics[width=0.085\linewidth]{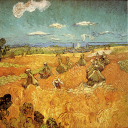}\\
    
    \raisebox{2mm}{\rotatebox[origin=c]{90}{\parbox{0.5cm}{\scriptsize \centering Generated}}}&
    \includegraphics[width=0.085\linewidth]{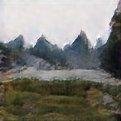}&
    \includegraphics[width=0.085\linewidth]{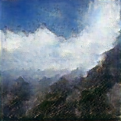}&
    \includegraphics[width=0.085\linewidth]{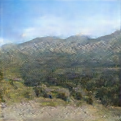}&
    \includegraphics[width=0.085\linewidth]{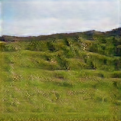}&
    \includegraphics[width=0.085\linewidth]{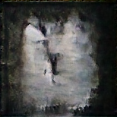}&
    \includegraphics[width=0.085\linewidth]{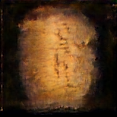}&
    \includegraphics[width=0.085\linewidth]{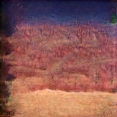}&
    \includegraphics[width=0.085\linewidth]{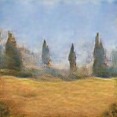}&
    \includegraphics[width=0.085\linewidth]{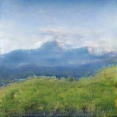}&
    \includegraphics[width=0.085\linewidth]{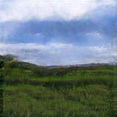}&
    \includegraphics[width=0.085\linewidth]{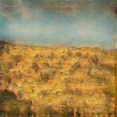}
    \end{tabular}

	\captionof{figure}{{\color{black} \textbf{Feedforward generation} with the Valley$_{5000}$ trained model on unseen images: \textbf{Left-side} - image form the same distribution. \textbf{Right-side} - images from completely different distributions.}}
    \label{fig:ood}
    
\end{figure}

\noindent\textbf{Feedforward generation\quad}
The meta learning approach, and the fact our method is able to learn from a relatively large dataset such as Valley$_{5000}$, introduce the ability to model a new image in one forward pass. 
Fig.~\ref{fig:inference} and Tab.~\ref{tab:datasets} show inference results of three different models trained on Valley$_{500}$,  Valley$_{2500}$ and  Valley$_{5000}$. As can be seen, it requires a significantly larger dataset than that of~\cite{shaham2019singan} to get the model to generalize. The network trained on Valley$_{5000}$ enables modeling of a new image in a fraction of a second, and results in coherent and reasonable generated images, compared to previous works, which are unable to perform this task.

\section{Limitations}

A prominent limitation of the method is the model's size. The hypernetwork approach induce large size projection layers, \eg a convolution layer with weight size of $(64, 64, 3, 3)$ and embedding size of $|e| = 512$ will result in a projection layer with weights size of $|(512, 36864)|\approx 18$M parameters. This obviously affect convergence, runtime and GPU memory usage.

In Fig.~\ref{fig:ood}, we quantitatively explore the  out-of-distribution generalization capabilities of our feedforward method when training on the Valley$_{5000}$ nature image dataset. As can be seen, for images that are completely out of domain, the generated images are not faithful to the input image. Training on a large-scale heterogeneous dataset to further improve generalization requires days of training. Until this experiment is performed, it is unclear whether the architecture has enough capacity to support this one-model-fits-all capability.

\section{Related work}
Hypernetworks, which were first introduced under this name in~\cite{ha2016hypernetworks}, are networks that generate the weights of a second {\em primary} network that computes the actual task. Hypernetworks are especially suited for meta-learning tasks, such as few-shot~\cite{bertinetto2016learning} and continual learning tasks~\cite{Oswald2020Continual}, due to the knowledge sharing ability of the weights generating network. Knowledge sharing in hypernetworks was recently used for continual learning by~\cite{Oswald2020Continual}.

Predicting the weights instead of performing backpropagation can lead to efficient neural architecture search~\cite{brock2018smash,zhang2018graph}, and hyperparameter selection~\cite{lorraine2018stochastic}. In~\cite{Littwin_2019_ICCV}, hypernetworks were applied for 3D shape reconstruction from a single image. In~\cite{sitzmann2020implicit} hypernetworks were shown to be useful for learning shared image representations. 
Note that while the name of the method introduced in~\cite{pmlr-v97-ratzlaff19a} is reminiscent of our method, it solves a different task with a completely different algorithm. Their method does not employ a hypernetwork to parameterize their generator (or discriminator), rather their generator serves as a hypernetwork itself. In addition, they intend to learn the distribution of weights of high-performing classifiers on a given classification task, which is a different application.

Several GAN-based approaches were proposed for learning from a single image sample. 
Deep Image Prior~\cite{DIP} and Deep Internal Learning~\cite{shocher2018zero}, showed that a deep convolutional network can form a useful prior for a single image in the context of denoising, super-resolution, and inpainting.  SinGAN~\cite{rottshaham2019singan} uses patch-GAN~\cite{pix2pix, rottshaham2019singan, markov1, markov2} to model the multiscale internal patch distribution of a single image, thus generating novel samples. ConSinGAN~\cite{consingan} extends SinGAN, improving the quality and train time. However, these methods need to be trained on each image individually. In this work, we propose a novel approach based on hypernetworks that leverages the capabilities of single-image generation and enables efficient training on an arbitrary sized dataset while keeping the unique properties of single-image training.

\section{Conclusions}
Given the abundance of unlabeled training images, training a single image GAN is unjustifiable, if viable multi-image alternatives exist. We present the first such alternative, which also opens the door to novel applications that are not possible with the existing models, such as the interpolation between single-image domains and feedforward modeling. From a technical perspective, we present the first, to our knowledge, adversarial hypernetwork. Working with this novel multi-network structure requires an understanding of the interplay between the involved components, and we support our method by a theoretical analysis.

\section*{Acknowledgments}
This project has received funding from the European Research Council (ERC) under the European Unions Horizon 2020 research and innovation programme (grant ERC CoG 725974). The contribution of the first author is part of a Master thesis research conducted at Tel Aviv University.

\bibliographystyle{abbrv}
\bibliography{hyper_bib}

\clearpage

\appendix

\addcontentsline{toc}{section}{Appendix} 
\part{Appendix} 
\begin{center}
	\includegraphics[scale=0.3]{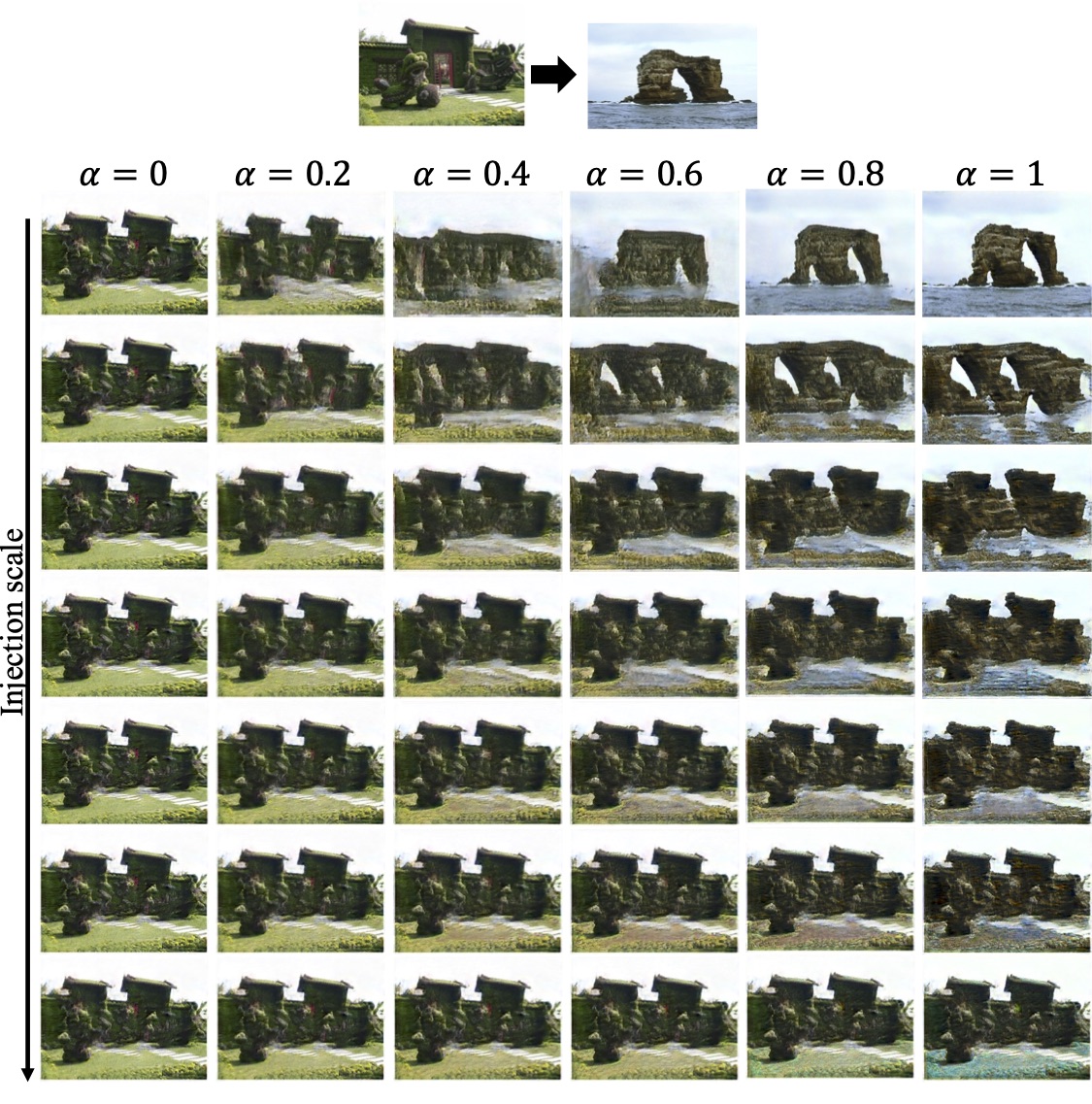}
\parttoc 
\end{center}

Due to file size constraints, all the figures in this manuscript were compressed to a lower resolution. For full resolution figures, please refer to the project site : \url{https://RaphaelBensTAU.github.io/MetaInternalLearning/}

\section{Theoretical Analysis}

\subsection{Shared Discriminator}
\label{sec:shared-disc}

In this section, we consider the case where the model has two main components: a hyper-generator $g_i(z,I) = g_i(z;f^i_g(I;\theta_{f_g}))$ along with a standard discriminator $d_i(u) = d_i(u;\theta_d)$ that is shared among all samples $I$, as illustrated in Fig.~\ref{fig:leakage}. In this setting, the adversarial objective function is defined as:
\begin{equation}
\begin{aligned}
\mathcal{L}_{adv}(g_i,d_i) := \mathbb{E}_{I} \big\{\mathop{\mathbb{E}}_{z_{1:i}} d_i(u_{\hat{I},i}) - \mathop{\mathbb{E}}_{u_{i,I}} d_i(u_{I,i}) \big\}
\end{aligned}
\end{equation}
The following proposition shows that if the expected (with respect to the distribution of $I$) distance between the distributions $\mathcal{D}_{\hat{I},i}$ and $\mathcal{D}_{i}$ is small, then, the loss function $\mathcal{L}_{adv}(g_i,d_i)$ tends to be small. For this purpose, we assume that $\mathcal{C}$ is closed under multiplication by positive scalars (i.e., $\alpha d \in \mathcal{C}$ for all $d\in \mathcal{C}$ and $\alpha > 0$). This is a technical assumption that holds for any set of neural networks, with a linear top-layer. 

\begin{proposition}
\label{prop:2}
Let $g_i(z,I) := g_i(z;f^i_g(I;\theta_{f_g}))$ be a hyper-generator and $d_i \in \mathcal{C}$ a shared discriminator at scale $i$. Let $\mathcal{D}_{i}$ be the distribution of $u \sim \mathcal{D}_{I,i}$, where $I \sim \mathcal{D}_I$. Assume that $\mathcal{C}$ is closed under multiplication by positive scalars. Then,
\begin{equation}
\mathcal{L}_{adv}(g_i,d_i) \leq \|d_i\|_{L} \cdot \mathbb{E}_I [W_{\mathcal{C}}(\mathcal{D}_{\hat{I},i}, \mathcal{D}_{i})]
\end{equation}
In particular, $\max_{d_i \in \mathcal{C}^1}\mathcal{L}_{adv}(g_i,d_i) \leq \mathbb{E}_I [W_{\mathcal{C}}(\mathcal{D}_{\hat{I},i}, \mathcal{D}_{i})]$.
\end{proposition}

\begin{proof}
Let $\alpha := \|d_i\|_L$, $u_{I,i} \sim \mathcal{D}_{I,i}$ (conditioned on a fixed $I$) and let $u \sim \mathcal{D}_{i}$ be a random variable. We can write: 
\begin{equation*}
\begin{aligned}
\mathcal{L}(g_i,d_i) \leq& \max_{d'_i\in \mathcal{C}^{\alpha}} \mathcal{L}(g_i,d'_i)\\
=& \max_{d'_i\in \mathcal{C}^{\alpha}} \mathop{\mathbb{E}}_{I} \left\{\mathop{\mathbb{E}}_{z_{1:i}} d'_i(u_{\hat{I},i}) - \mathop{\mathbb{E}}_{u_{I,i}} d'_i(u_{I,i}) \right\} \\
=& \max_{d'_i\in \mathcal{C}^{\alpha}} \left\{\mathop{\mathbb{E}}_{I_1}\mathop{\mathbb{E}}_{z_{1:i}} d'_i(u_{\hat{I},i}) - \mathop{\mathbb{E}}_{I_2}\mathop{\mathbb{E}}_{u_{I,i}} d'_i(u_{I,i}) \right\} \\
=& \max_{d'_i\in \mathcal{C}^{\alpha}} \left\{\mathop{\mathbb{E}}_{I_1}\mathop{\mathbb{E}}_{z_{1:i}} d'_i(u_{\hat{I},i}) - \mathop{\mathbb{E}}_{u} d'_i(u) \right\} \\
=& \max_{d'_i\in \mathcal{C}^{\alpha}} \mathop{\mathbb{E}}_{I}\left\{\mathop{\mathbb{E}}_{z_{1:i}} d'_i(u_{\hat{I},i}) - \mathop{\mathbb{E}}_{u} d'_i(u) \right\}, \\
\end{aligned}
\end{equation*}
where $I_1,I_2 \sim \mathcal{D}_I$ are two i.i.d. random variables. We note that for any real-valued function $f$, we have: $\max_{x} \mathbb{E}_y[f(x,y)] \leq \mathbb{E}_{y}[\max_x f(x,y)]$.  Therefore, 
\begin{equation}
\begin{aligned}
\mathcal{L}(g_i,d_i) \leq&  \mathop{\mathbb{E}}_{I} \max_{d'_i \in \mathcal{C}^{\alpha}} \left\{ \mathop{\mathbb{E}}_{z_{1:i}} d'_i(u_{\hat{I},i}) - \mathop{\mathbb{E}}_{u} d'_i(u) \right\} \\
\end{aligned}
\end{equation}
We note that any function $d \in \mathcal{C}^1$ can be translated into a function $\alpha d \in \mathcal{C}^{\alpha}$ and vice versa since $\mathcal{C} = \alpha \cdot \mathcal{C}$. In particular, $\mathcal{C}^{\alpha} = \alpha \cdot \mathcal{C}^1$. Hence, we have:
\begin{equation}
\begin{aligned}
\mathcal{L}(g_i,d_i) 
\leq&  \mathop{\mathbb{E}}_{I} \max_{d'_i \in \mathcal{C}^{1}} \left\{ \mathop{\mathbb{E}}_{z_{1:i}} \alpha \cdot d'_i(u_{\hat{I},i}) - \mathop{\mathbb{E}}_{u} \alpha \cdot d'_i(u) \right\} \\
=& \mathop{\mathbb{E}}_{I} [\alpha \cdot W_{\mathcal{C}}(\mathcal{D}_{\hat{I},i},\mathcal{D}_{i})] \\
=& \alpha \mathbb{E}_{I} [ W(\mathcal{D}_{\hat{I},i},\mathcal{D}_{i})],
\end{aligned}
\end{equation}
which proves the claim.
\end{proof}

This proposition shows that a hyper-generator $g_i(\cdot,I)$ that generates samples $\hat{I}_i$ whose patches are similar to $\mathcal{D}_{i}$ would minimize the loss function $\mathcal{L}_{adv}(g_i,d_i)$, even though the generated samples are not conditioned on the image $I$. Therefore, minimizing the adversarial loss with a shared discriminator does not guarantee that $g_i(\cdot,I)$ would generate samples $\hat{I}_i$ that are similar to $I_i$, which is undesirable.

\subsection{Shared Feature Extractor}
\label{sec:feature}

We note that as a strategy one could reduce the number of trainable parameters in the whole model, by restricting $f_g$ and $f_d$ to share their encoding component $e$, as illustrated in Fig.~\ref{fig:shared_arch}. In this section, we show two failing cases of this approach. First, we consider the case where $\theta_e$ is optimized to minimize the objectives of both $g$ and $d$. As a second case, we consider the case where $\theta_e$ is optimized to minimize the objective of $g$.

\noindent\textbf{Case 1\quad} We first consider the case where the model is trained using GD, when $f_g$ and $f_d$ share their representation function's weights. Specifically, GD iteratively updates $(\theta_e,\theta^i_{p_g})$ to minimize $\mathcal{L}_{adv}(g_i,d_i)$ and updates $(\theta_e,\theta^i_{p_d})$ to maximize $\mathcal{L}_{adv}(g_i,d_i) - \lambda_1 \cdot \mathcal{L}_{lip}(d_i)$. We denote this optimization process by $\mathbb{A}$. The following proposition shows that $\theta_e$ is trained to minimize $\mathcal{L}_{lip}(d_i)$ only and that $\mathbb{A}$ suffers from a wide span of undesirable equilibrium points. 

\begin{proposition} Let $g_i(z,I) := g_i(z;f^i_g(I;\theta_{f_g}))$ and $d_i(u,I) := d_i(u;f^i_d(I;\theta_{f_d}))$ be the hyper-generator and the hyper-discriminator, with an activation function $\sigma$ that satisfies $\sigma(0)=0$. Assume that $\theta_{e_g} = \theta_{e_d}$ is shared among  $f^i_g$ and $f^i_d$. Then, $\mathbb{A}$ trains $e_g=e_d$ to minimize $\mathcal{L}_{lip}(d_i)$ only. In addition, let $(\theta_e,\theta_{p_g},\theta_{p_d})$ be a set of parameters with $E_g = E_d \equiv 0$. Then, $(\theta_e,\theta_{p_g},\theta_{p_d})$ is an equilibrium point of $\mathbb{A}$.
\end{proposition}

\begin{proof}
We denote $e=e_g=e_d$. Let $\theta_{e}$, $\theta^i_{p_g}$ and $\theta^i_{p_d}$ be the parameters of $e$, $p_g$ and $p^i_d$. 
Each iteration of GD updates the weights $(\theta_{e},\theta^i_{p_d})$ of $d_i$ with the following step: $-\mu  \frac{\partial\mathcal{L}_{adv}(g_i,d_i)}{\partial (\theta_{e},\theta^i_{p_d})} + \mu \frac{\partial \mathcal{L}_{lip}(d_i)}{\partial (\theta_{e},\theta^i_{p_d})}$. On the other hand, the GD step for the weights $(\theta_{e},\theta^i_{p_g})$ of $g_i$ is $+\mu  \frac{\partial\mathcal{L}_{adv}(g_i,d_i)}{\partial (\theta_{e},\theta^i_{p_d})}$. Therefore, since $d_i$ and $g_i$ share weights within their representation function $e$, its update is the sum of the two steps $-\mu  \frac{\partial\mathcal{L}_{adv}(g_i,d_i)}{\partial \theta_{e}}$ and $+\mu  \frac{\partial\mathcal{L}(g_i,d_i)}{\partial \theta_{e}}$ and $-\mu \frac{\partial \mathcal{L}_{lip}(d_i)}{\partial \theta_{e}}$, which is simply $-\mu \frac{\partial \mathcal{L}_{lip}(d_i)}{\partial \theta_{e}}$. Therefore, $e$ is trained to minimize $\mathcal{L}_{lip}(d_i)$ using GD.

To see why $(\theta_e,\theta^i_{p_g},\theta^i_{p_d})$ (with $E_g\equiv 0$) is an equilibrium point, we notice that $d_i\equiv 0$ is a global minima of $\mathcal{L}_{lip}(d_i)$. In particular, $\theta_{e}$ would not change when applying $\mathbb{A}$. In addition, we note that if $E_g(I) = E_d(I) = 0$, then, the outputs of $E_g$, $E_d$, $f^i_g$, $f^i_d$, $g_i$ and $d_i$ are all zero, regardless of the values of the weights $\theta^i_{p_g},\theta^i_{p_d}$, because $\sigma(0)=0$. Therefore, the gradients of $\mathcal{L}_{adv}(g_i,d_i)$ with respect to $\theta^i_{p_g}$ and $\theta^i_{p_d}$ are zero, and we conclude that $\theta^i_{p_g}$ and $\theta^i_{p_d}$ would not update as well. 
\end{proof}

\noindent\textbf{Case 2 \quad} As an additional investigation, we consider the case where GD iteratively optimizes $(\theta_e,\theta^i_{p_g})$ to minimize $\mathcal{L}_{adv}(g_i,d_i)$, $\theta^i_{p_d}$ to maximize $\mathcal{L}_{adv}(g_i,d_i)$ and $(\theta_e,\theta^i_{p_g})$ to minimize the loss $\mathcal{L}_{lip}(d_i)$. We denote this optimization process by $\mathbb{B}$. The following proposition shows that this procedure suffers from a wide span of undesirable equilibrium points.  

\begin{proposition}
Let $g_i(z,I) := g_i(z;f^i_g(I;\theta_{f_g}))$ and $d_i(u,I) := d_i(u;f^i_d(I;\theta_{f_d}))$ be a hyper-generator and a hyper-discriminator, both with activation functions $\sigma$ that satisfy $\sigma(0)=0$. Then, any set of parameters $(\theta_e=0,\theta_{p_g},\theta_{p_d})$ is an equilibrium point of $\mathbb{B}$.
\end{proposition}

\begin{proof}
We note that if $\theta_e = 0$, then, since $\sigma(0)=0$, the outputs of $e$, $f^i_g$, $f^i_d$, $g_i$ and $d_i$ are all zero, regardless of the values of the weights $\theta^i_{p_g},\theta^i_{p_d}$. In particular, the gradients of $\mathcal{L}_{adv}(g_i,d_i)$ with respect to $\theta^i_{p_g}$ and $\theta^i_{p_d}$ are zero. In addition, the Lipschitz loss function is at its global minima for $d_i$, and therefore, its gradient with respect to $(\theta_e,\theta_{p_d})$ is zero as well. Therefore, we conclude that any possible step starting from $(\theta_e = 0,\theta_{p_g},\theta_{p_d})$ would not change the weights.
\end{proof}

\subsection{Our Method}

\begin{proposition} 

Assume that $\mathcal{I} \subset \mathbb{R}^{3 \times h \times w}$ is compact. Let $\epsilon>0$ be an approximation error. Let $g_i(z,I) := g_i(z;f^i_g(I;\theta_{f_g}))$ be a hyper-generator and $\mathcal{C}$ a class of discriminators. Assume that $S^*$ is continuous over $\mathcal{I}$. Then, there is a large enough neural network $f^i_d$ (whose size depends on $\epsilon$), such that, the hyper-discriminator $d_i(u,I) := d_i(u;f^i_d(I;\theta_{f_d}))$ satisfies:
\begin{equation*}
\begin{aligned}
\mathbb{E}_I W_{\mathcal{C}}(\mathcal{D}_{\hat{I},i},\mathcal{D}_{I,i}) =& \max_{\theta_{f_d}} \mathop{\mathbb{E}}_I\left\{\mathop{\mathbb{E}}_{z_{1:i}} d_i(u_{\hat{I},i};f^i_d(I)) - \mathop{\mathbb{E}}_{u_{I,i}} d_i(u_{I,i};f^i_d(I)) \right\} + o_{\epsilon}(1),
\end{aligned} 
\end{equation*}
where the maximum is taken over the parameterizations $\theta_{f_d}$ of $f_{d}$, such that, $d_i(\cdot; f^i_d(I;\theta_{f_d})) \in \mathcal{C}^1$.
\end{proposition}
\begin{proof}
Let $\mathcal{S}^1$ be the set of functions $S:I \mapsto \theta_I$, where $\theta_I$ correspond to a discriminator $d_i(\cdot;\theta_I) \in \mathcal{C}^1$. Let $\mathcal{Q}$ be the set of parameters $\theta_{f_d}$, such that, $d_i(\cdot; f^i_d(I;\theta_{f_d})) \in \mathcal{C}^1$ for all $I \in \mathcal{I}$. In particular, for any $\theta_{f_d} \in \mathcal{Q}$, we have: $f^i_d(I;\theta_{f_d})) \in \mathcal{S}^1$. Hence, we have: 
\begin{equation*}
\begin{aligned}
\mathop{\mathbb{E}}_I W_{\mathcal{C}}(\mathcal{D}_{\hat{I},i},\mathcal{D}_{I,i}) =& \max_{S \in \mathcal{S}^1} \mathop{\mathbb{E}}_I\left\{\mathop{\mathbb{E}}_{z_{1:i}} d_i(u_{\hat{I},i};S(I)) - \mathop{\mathbb{E}}_{u_{I,i}} d_i(u_{I,i};S(I)) \right\} \\
\geq& \max_{\theta_f \in \mathcal{Q}} \mathop{\mathbb{E}}_I\left\{\mathop{\mathbb{E}}_{z_{1:i}} d_i(u_{\hat{I},i};f^i_d(I)) - \mathop{\mathbb{E}}_{u_{I,i}} d_i(u_{I,i};f^i_d(I)) \right\}, \\
\end{aligned} 
\end{equation*}
Next, we would like to prove the opposite direction. Let $S^*$ be a continuous maximizer of the following objective:
\begin{equation}
\max_{S \in \mathcal{S}^1} \mathop{\mathbb{E}}_I\left\{\mathop{\mathbb{E}}_{z_{1:i}} d_i(u_{\hat{I},i};S(I)) - \mathop{\mathbb{E}}_{u_{I,i}} d_i(u_{I,i};S(I)) \right\} 
\end{equation}
Since $\mathcal{I}$ is compact, by~\cite{hanin2017approximating} there is a large enough neural network $f^i_d(\cdot;\theta^*_{f_d})$ (with sigmoid/tanh/ReLU activation) that approximates the continuous function $S^*$ up to an approximation error $\epsilon$ (of our choice) with respect to the $L_{\infty}$ norm, i.e., $\|f^i_d(\cdot;\theta^*_{f_d}) - S^*\|_{\infty} \leq \epsilon$. 

Recall that $\mathcal{D}_{\hat{I},i}$ and $\mathcal{D}_{I,i}$ are supported by compact sets. In addition, since $S^*$ is continuous over a compact set, $S^*(\mathcal{I})$ is compact as well. Let $U$ be a compact set that contains the union of the supports of both $\mathcal{D}_{\hat{I},i}$ and $\mathcal{D}_{I,i}$. Let $V$ be a compact set that contains $S^*(\mathcal{I})$ and $f^i_d(\mathcal{I};\theta^*_{f_d})$. Since the discriminator $d_i(u;\theta^i_{d})$ is a continuous function (a neural network with continuous activation functions) with respect to both $(u,\theta^i_d)$, it is uniformly continuous over $U \times V$. Therefore, we have:
\begin{equation}
\sup_{u \in U,I \in \mathcal{I}} \Big\vert d(u;S(I)) - d(u;f^i_d(I;\theta^*_{f_d}))\Big\vert = o_{\epsilon}(1)
\end{equation}
In particular, we have:
\begin{equation*}
\begin{aligned}
&\max_{S \in \mathcal{S}^1} \mathop{\mathbb{E}}_I\left\{\mathop{\mathbb{E}}_{z_{1:i}} d_i(u_{\hat{I},i};S(I)) - \mathop{\mathbb{E}}_{u_{I,i}} d_i(u_{I,i};S(I)) \right\} \\
=&\mathop{\mathbb{E}}_I\left\{\mathop{\mathbb{E}}_{z_{1:i}} d_i(u_{\hat{I},i};S^*(I)) - \mathop{\mathbb{E}}_{u_{I,i}} d_i(u_{I,i};S^*(I)) \right\} \\
\leq& \mathop{\mathbb{E}}_I\left\{\mathop{\mathbb{E}}_{z_{1:i}} d_i(u_{\hat{I},i};f^i_d(I;\theta^*_{f_d})) - \mathop{\mathbb{E}}_{u_{I,i}} d_i(u_{I,i};f^i_d(I;\theta^*_{f_d})) \right\} + o_{\epsilon}(1) \\
\leq&\max_{\theta_{f_d} \in \mathcal{Q}} \mathop{\mathbb{E}}_I\left\{\mathop{\mathbb{E}}_{z_{1:i}} d_i(u_{\hat{I},i};f^i_d(I)) - \mathop{\mathbb{E}}_{u_{I,i}} d_i(u;f^i_d(I)) \right\} + o_{\epsilon}(1) \\
\end{aligned} 
\end{equation*}
which completes the proof.
\end{proof}

\section{Training}
\subsection{Architecture}
We use ResNet-34 for the hypernetworks of both the generator and discriminator, with an embedding size of size 512. A multi-head dense linear layer is then applied and projects the image embedding into the different convolutional blocks of the main network. The main networks, (i.e., the generator and discriminator) share the same architecture and consist of 5 conv-blocks per scale of the form Conv(3 x 3)-LeakyRelu with 64 kernels per block.
For the generator, we hold a set of 10 linear heads projections for each scale, where each projection outputs the weights (or the biases) of its respective scale in the generator. For the discriminator, when training, only the current scale's linear projections are needed, thus we hold a single set of 10 linear head projections.\\ 
All LeakyReLU activations have a slope of 0.02 for negative values except when we use a classic discriminator for single image training, for which we use a slope of 0.2. Additionally, the generator’s last conv-block activation at each scale is Tanh instead of ReLU and the discriminator’s last convolutional block at each scale does not include any activation.

Differently from \cite{shaham2019singan}, but similarly to \cite{consingan} we do not gradually increase the number of kernels during training. 

Grouped convolutions were used in order to perform parallel computations in the main network for each image with its respective weights to speed up training.

\subsection{Progressive training}
We train with an initial noise of width $s_0=25$ pixels, except when training on the 50-images dataset and for the single mini-batch experiment (for which we use $s_0=28$ and $s_0=27$ respectively) such that the dimensions of the initial noise are ($\lceil{s_0*ar}\rceil, s_0$) where ar is the aspect ratio of the image.
If trained with multiple images, the default aspect ratio used for training is 3/4.\\
In terms of sizes of the images processed at each scale, we progress in a geometrical way as \cite{shaham2019singan}  with a scale factor of $r=0.6$ i.e., at each scale $i>0$ ,images of size $s_i=\frac{s_{i-1}}{r}$ are processed. This results in 7 scales for an image of size 256.
Although we train in a progressive manner, our hypernetworks receives as input (128,128) constant sized versions of the real images regardless of the current scale processed.
We progress from a scale to another at the end of the training of a current scale in the generator by copying the weights of its 10 linear projections to the next scale's projections and freeze the current scale, for the discriminator, we simply copy the weights of its linear projections and can safely delete its current set of linear projections from memory.

\subsection{Optimization}
The loss function is minimized using Adam optimizer with momentum parameters $\beta_1=0.5$, $\beta_2=0.999$ and different learning rates for each training setting, which we decreased by a factor of 0.1 after 80\% of the iterations.  
We used $\lambda_1=0.1$ and $\lambda_2=50$ for the coefficient of the gradient penalty in WGAN and the reconstruction loss respectively, $\lambda_2=10$ can also be used and yield good results. We clip the gradient s.t it has a maximal L2 norm of 1 for both the generators and discriminator. Batch sizes of 16 were used for all experiments involving a dataset of images.

\begin{tabular*}{\linewidth}{@{\extracolsep{\fill}}lcc}
\toprule
                                    & lr$_g$ & lr$_d$  \\
\midrule
 Single image & 1e-5 & 5e-4 \\
 Single mini-batch & 1e-5  & 1e-5\\
 Dataset & 5e-5 & 5e-5 \\
 \bottomrule
\end{tabular*}

Similarly to \cite{shaham2019singan}, we use MSE as the reconstruction loss, and at each iteration we multiply each noise map $z_i$ for ($i>1)$, by the RMSE obtained. This results in zero-mean and MSE varianced gaussian distributed noise maps and indicates the amount of details that need to be added at that scale for the current batch.
For the reconstruction, a single $z_1^0$ fixed random noise is used for all the images. \\
For feedforward modeling and applications, we use this single fixed random noise, and we multiply each scale's intermediate noise map by the RMSE obtained at the last iteration of this same scale during training.

\subsection{Number of iterations}

\begin{tabular*}{\linewidth}{@{\extracolsep{\fill}}lcc}
\toprule
                                    & number of iterations by scale   \\
\midrule
 Single image & 1500-2000  \\
 Single mini-batch & 2000  \\
 Places-50 & 4000  \\
 LSUN-50 & 5000 \\
 C250 & 20000 \\
 C500 & 30000 \\
 V500 & 25000  \\
 V2500 & 100000  \\
 V5000 & 150000 \\
 \bottomrule
\end{tabular*}

\subsection{GPU usage for training models}

\begin{tabular*}{\linewidth}{@{\extracolsep{\fill}}lcc}
\toprule
                  & GPU memory usage (256x256 resolution)   \\
\midrule
  Single image & 11GB  \\
 Single mini-batch & 11GB-15GB  \\
 Datasets & 22GB  \\
 \bottomrule
\end{tabular*}

At test time, the GPU memory usage is significantly reduced and requires 5GB. We trained all of our single image models and baselines with a single 12GB GeForce RTX 2080. The other models were trained on a single 32GB Tesla V100. Notice we compared the single and the dataset runtimes in Table 1 in the main paper, by approximating the runtime on a GeForce, training until scale 5 on GeForce RTX 2080 (until 12GB is out of memory) and by taking in account the difference in power between the latter and V100 GPU.

\section{Training with a pretrained image encoder}
In this section, we consider training our method with a 
"frozen" pretrained ResNet34 i.e., optimizing only the linear projections.

Our method uses  single linear layer projections, which strongly restricts the expressiveness of our network if the image encoder is “frozen”. We thus experimented with increasing the depth of these projection networks.

Below are the results on Places-50 :

\begin{center}
\begin{tabular*}{\linewidth}{@{\extracolsep{\fill}}lcccc}
\toprule
                  & End-to-end (our setting) & 1 layer & 3 layer & 5 layer   \\
\midrule
 SIFID & 0.05  & 0.26 & 0.14 & 0.17 \\
 mSIFID & 0.07 & 0.56 & 0.23 & 0.27 \\
 Diversity & 0.50 & 0.79 & 0.63 & 0.63\\
 \bottomrule
\end{tabular*}
\end{center}

If the problem could be learned with a "small enough" depth, our method would benefit from even faster training, at the cost of enlarging the model (and its consequences on inference time). Even though the results are convincing (both visually and quantitatively) in favor of end-to-end training, we prefer not to reject the hypothesis that proper hyper-parameter tuning and perhaps some adaptations could lead to decent results with a frozen backbone.

\section{Single-Image Generation}

\begin{figure}[t]
    \setlength{\tabcolsep}{1pt} 
    \renewcommand{\arraystretch}{1} 
    \centering
    \begin{tabular}{c c}
        \begin{tabular}{c}
		    \includegraphics[width=0.45\linewidth]{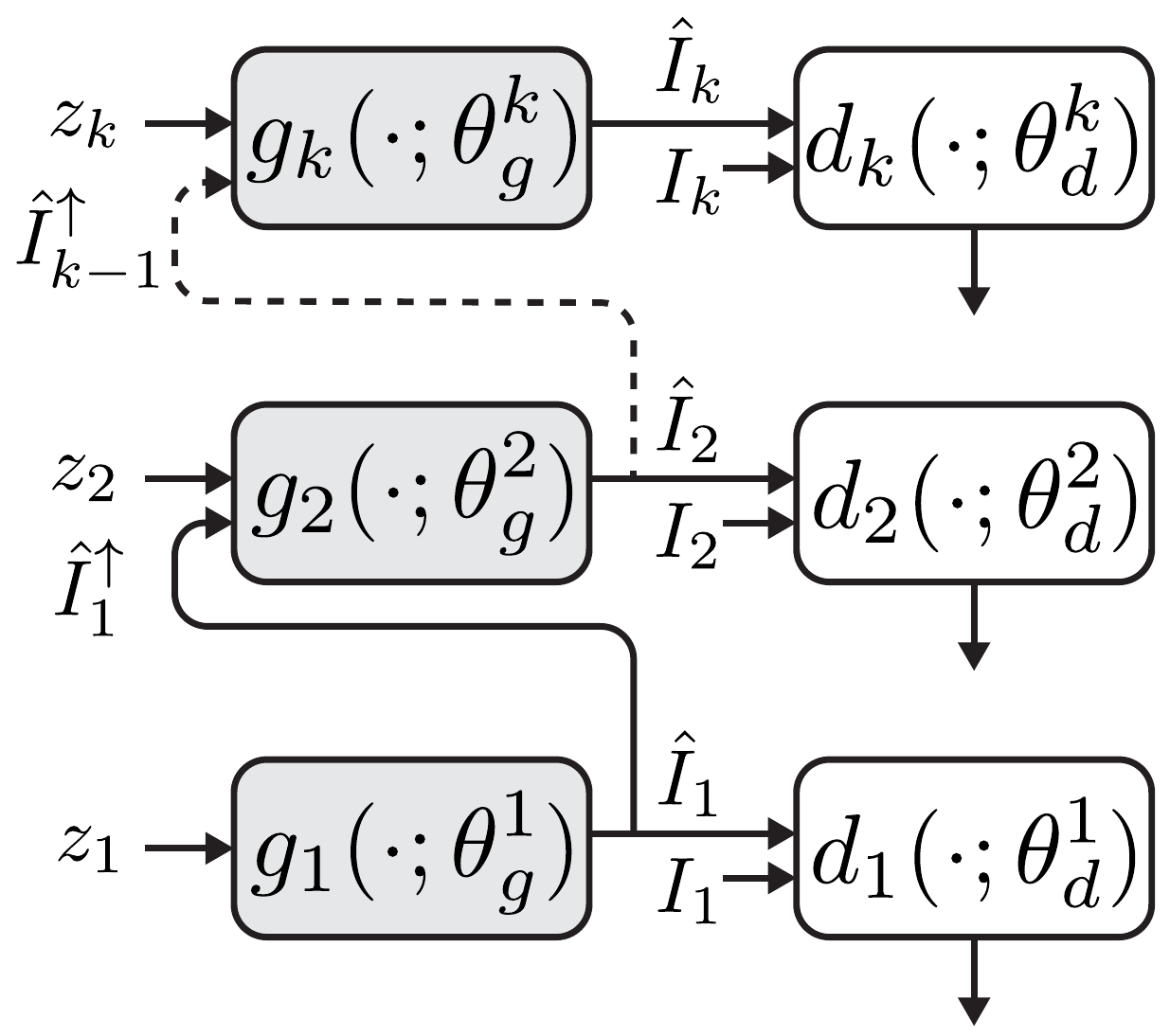}
        \end{tabular}&
        \begin{tabular}{c}
		    \includegraphics[width=0.45\linewidth]{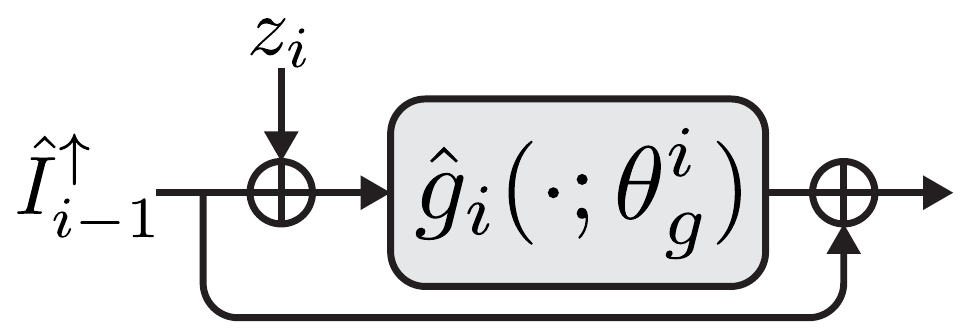}
        \end{tabular}\\
    (a) & (b)
    \end{tabular}
    \caption{{\bf Single-Image model architecture.} {\bf (a)} The complete hierarchical structure of generators and discriminators. {\bf (b)} The inner architecture of $g_i$, consists of noise addition and residual connection.}
    \label{fig:singan_arch}
\end{figure}
Fig.~\ref{fig:singan_arch} illustrates the single-image architecture with the internal skip connection, of~\cite{shaham2019singan}, as we discuss in section 2.

\subsection{Places-50 real images}
\includegraphics[width=\linewidth]{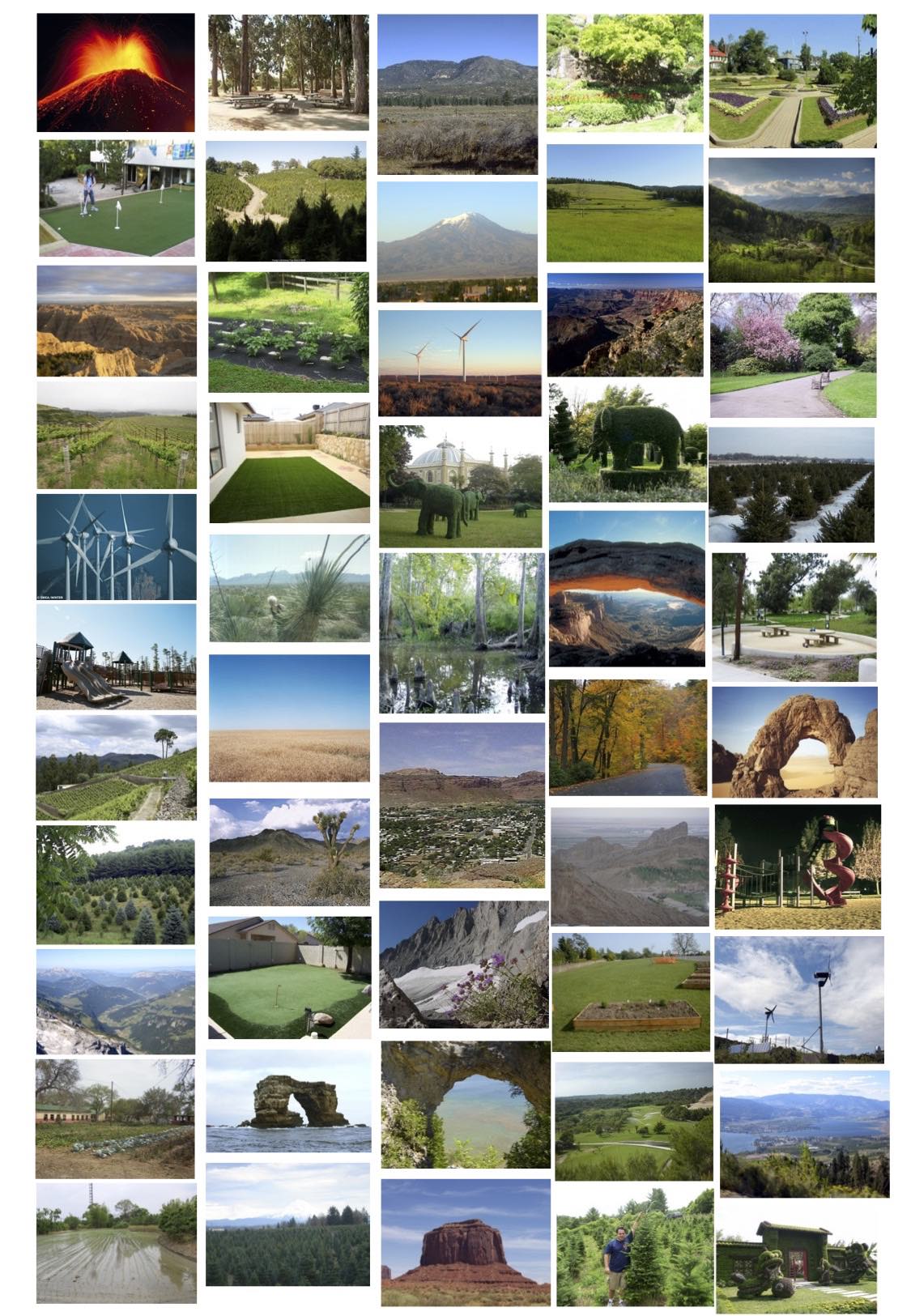} 
\newpage
\subsection{Places-50 fake images (single training)}
\includegraphics[width=\linewidth]{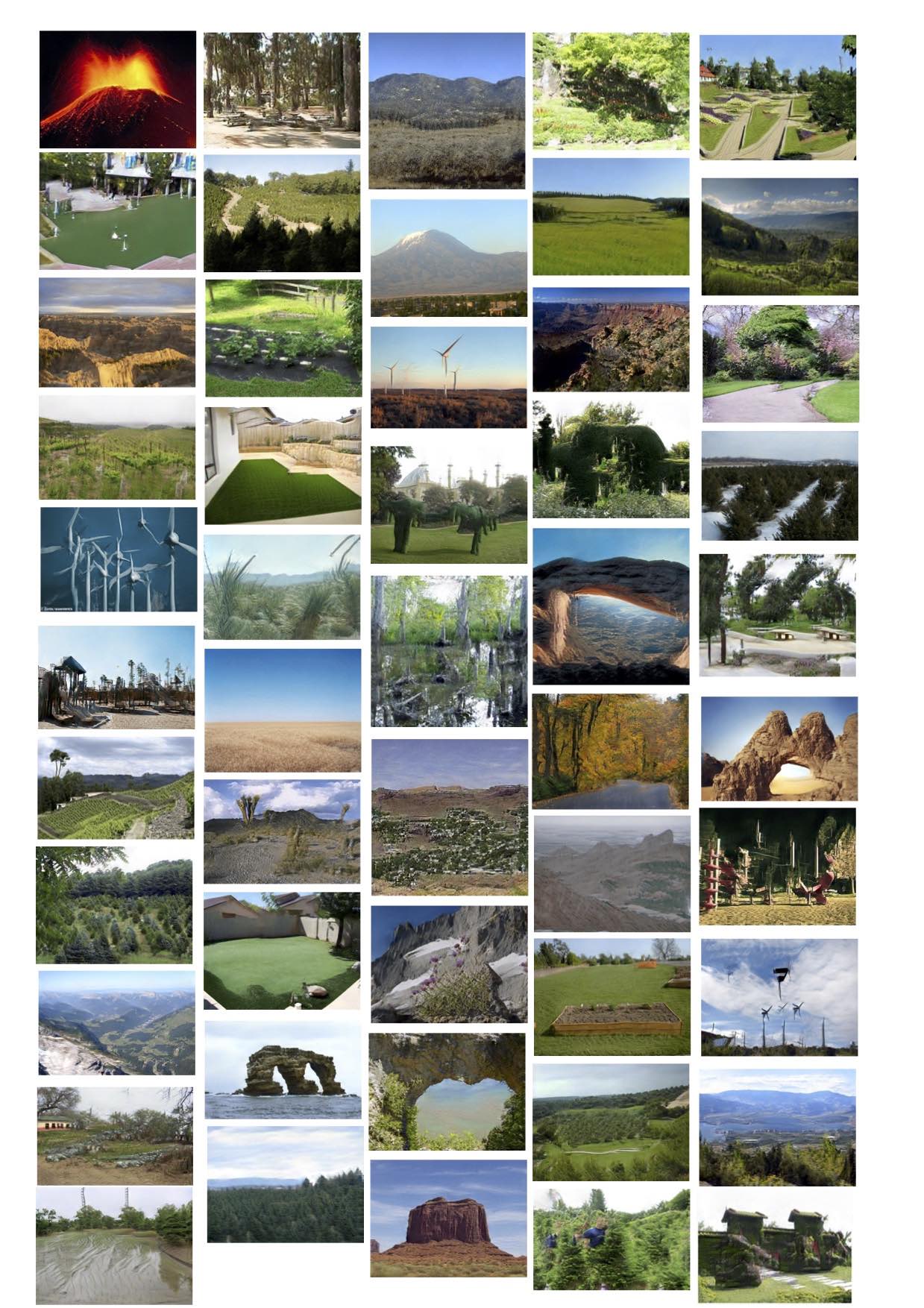}
\newpage

\subsection{Places-50 fake images (dataset training)}
\includegraphics[width=\linewidth]{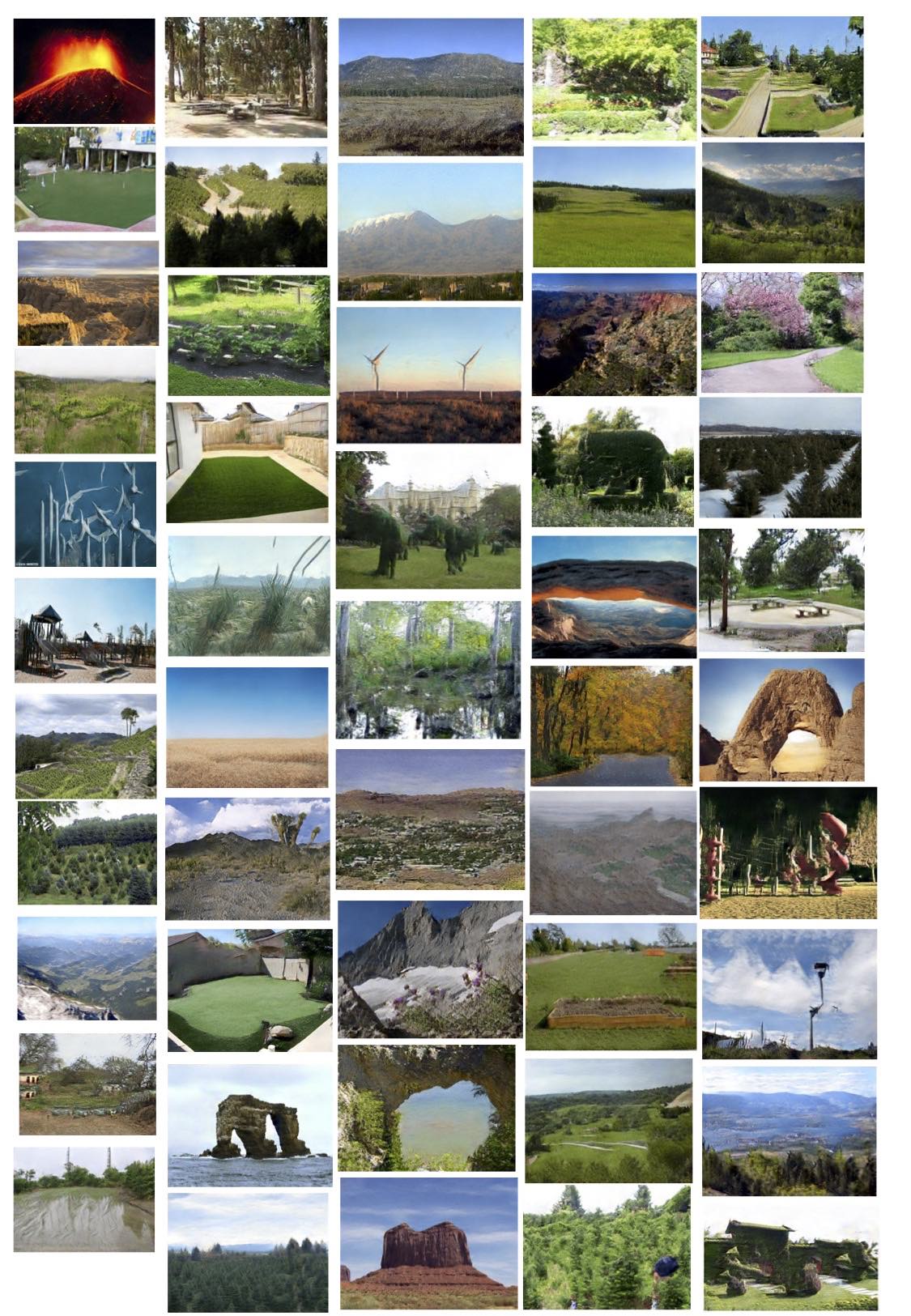}
\newpage

\subsection{Places-50 random samples (Single vs dataset training)}
\includegraphics[width=\linewidth]{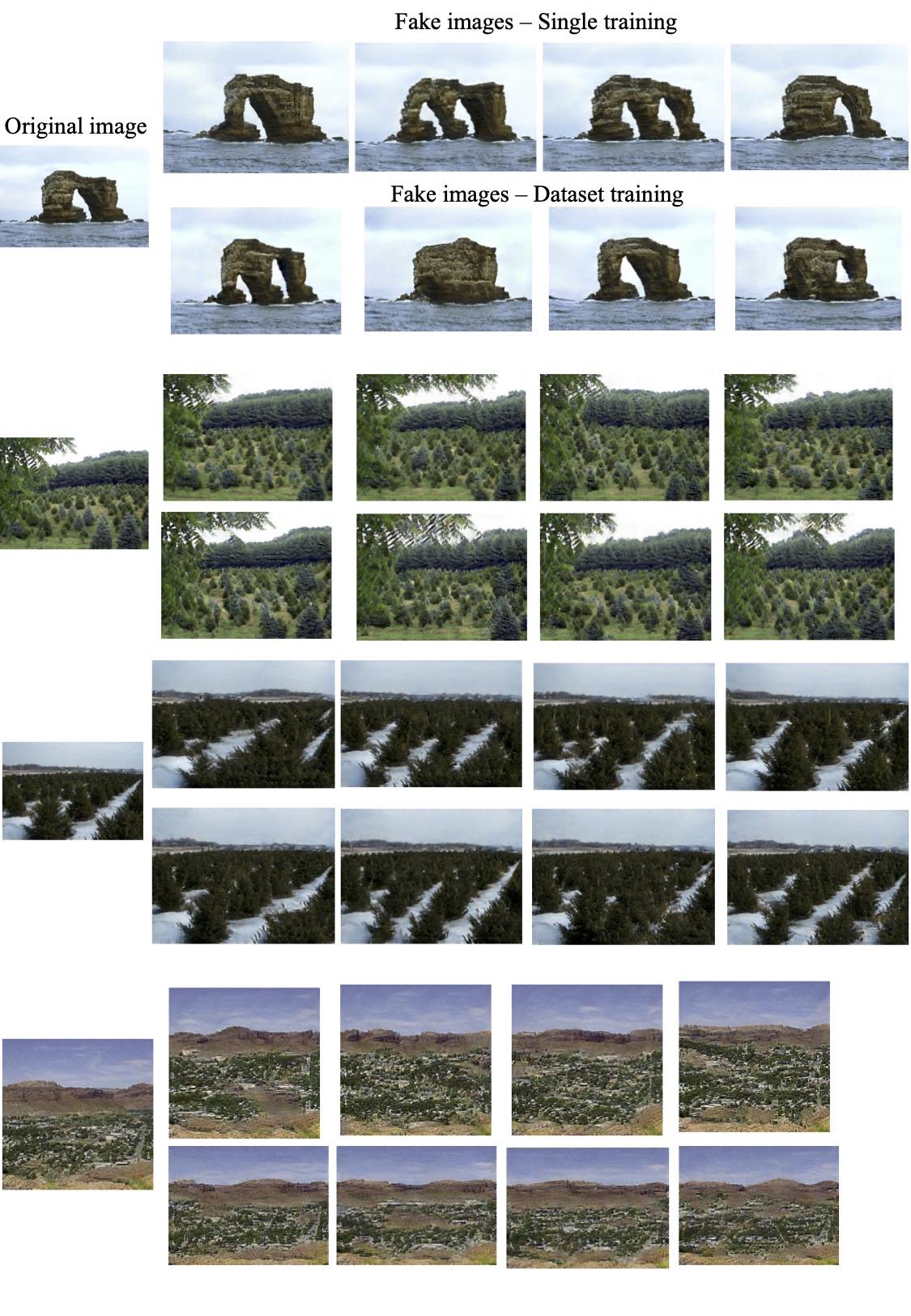}
\newpage

\subsection{Single mini-batch training experiment}
\includegraphics[width=\linewidth]{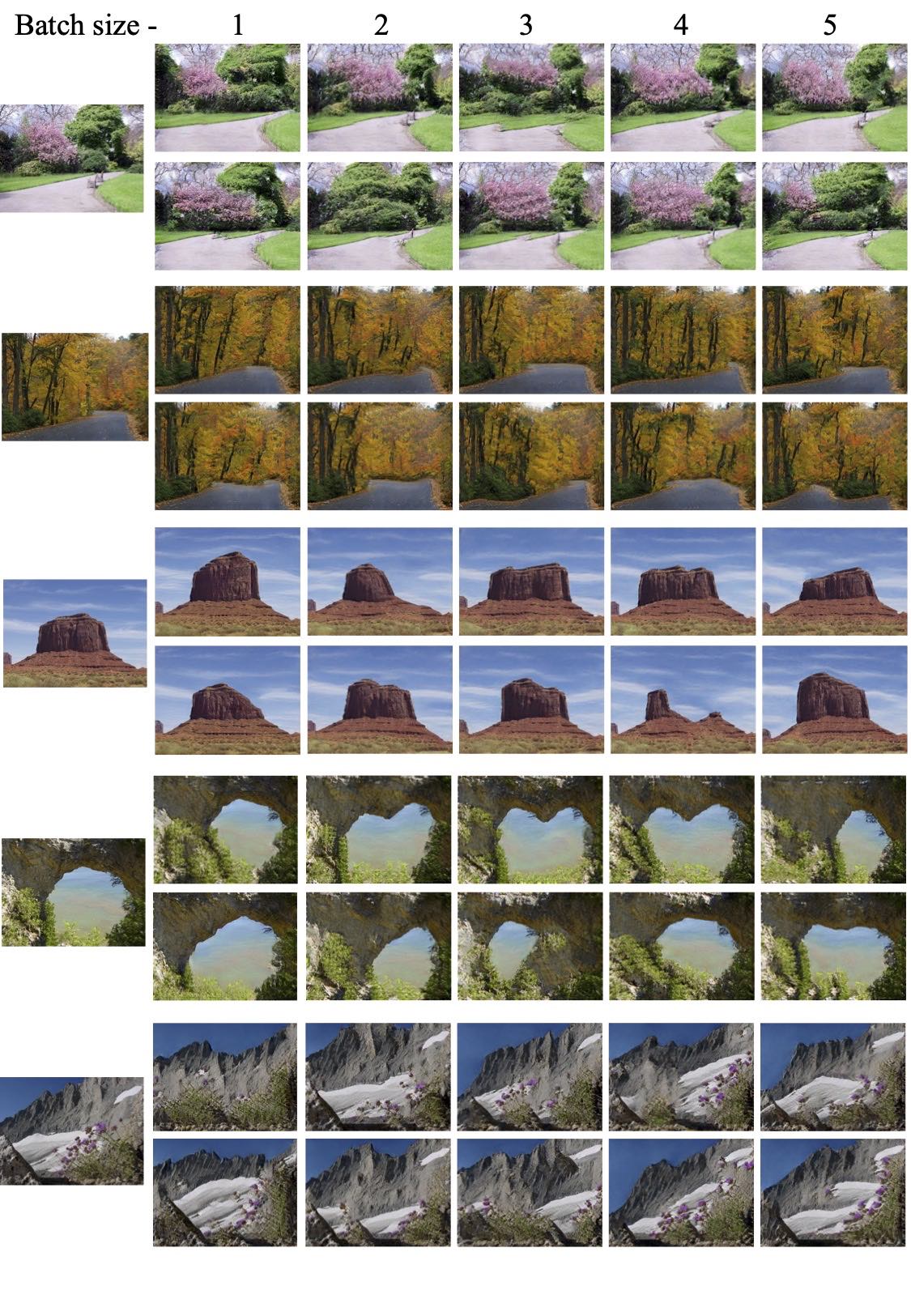}
\newpage

\subsection{V500 - Original images and random samples}
\includegraphics[width=\linewidth]{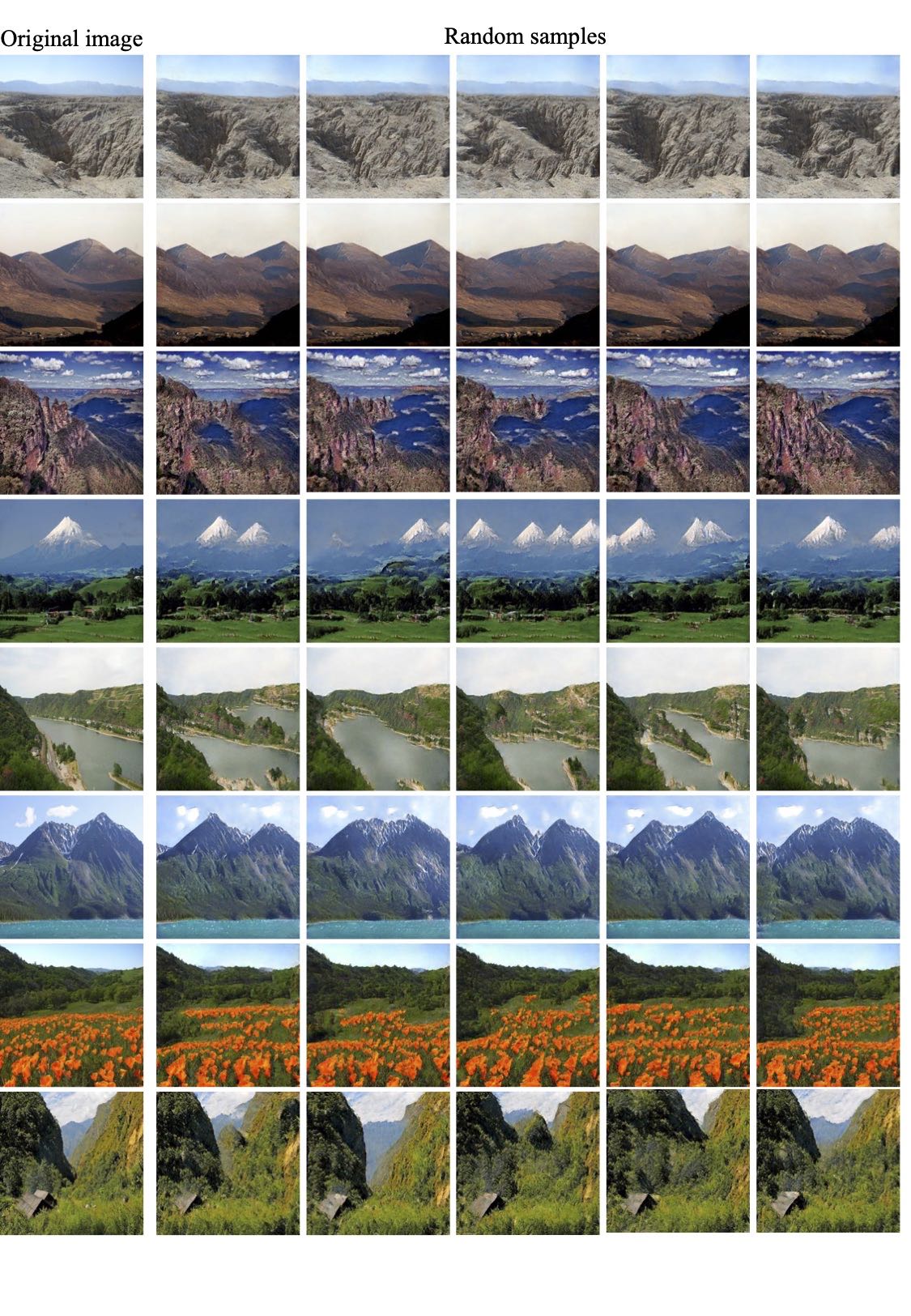}
\newpage

\subsection{V2500 - Original images and random samples} 
\includegraphics[width=\linewidth]{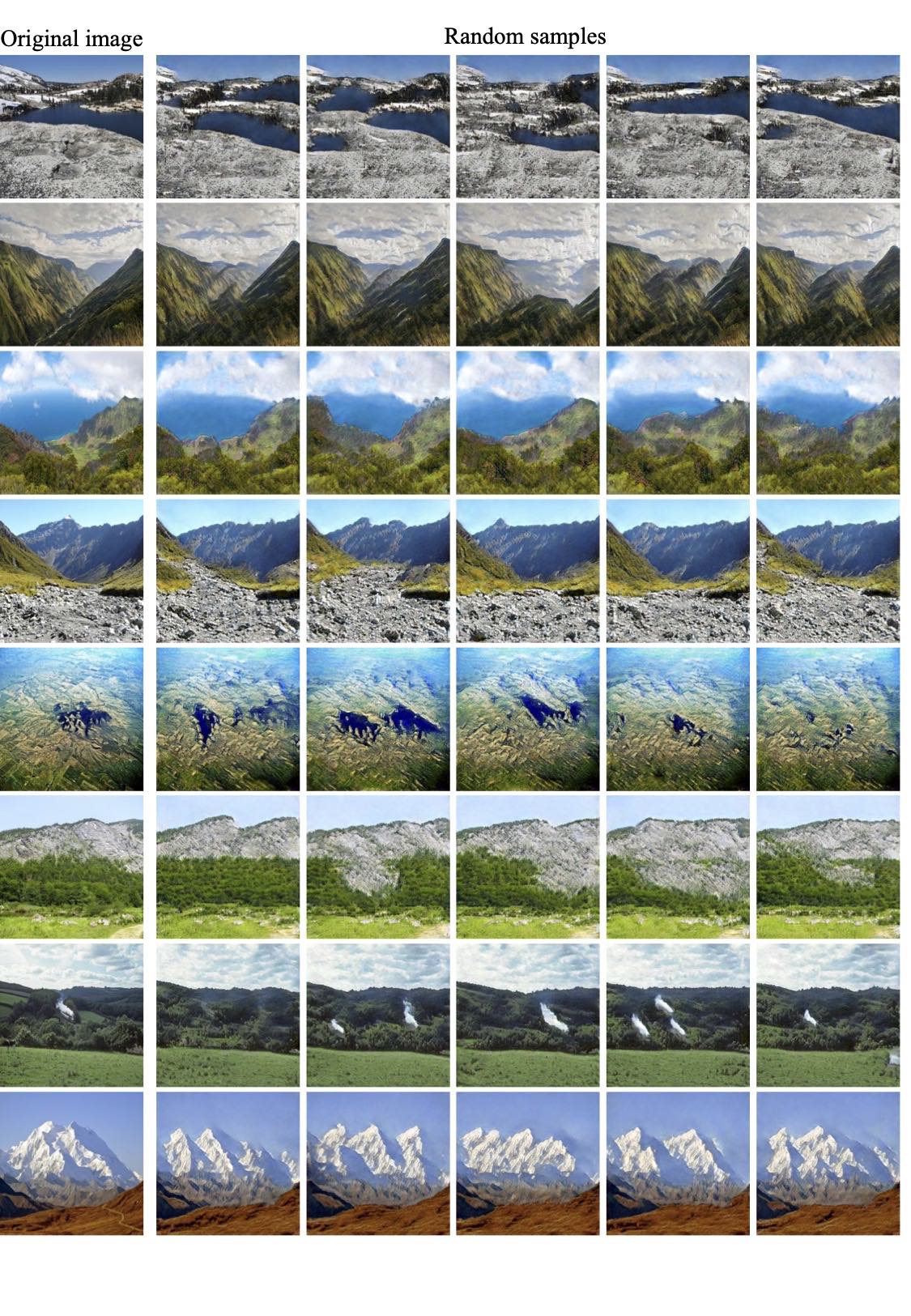}
\newpage

\subsection{V5000 - Original images and random samples} 
\includegraphics[width=\linewidth]{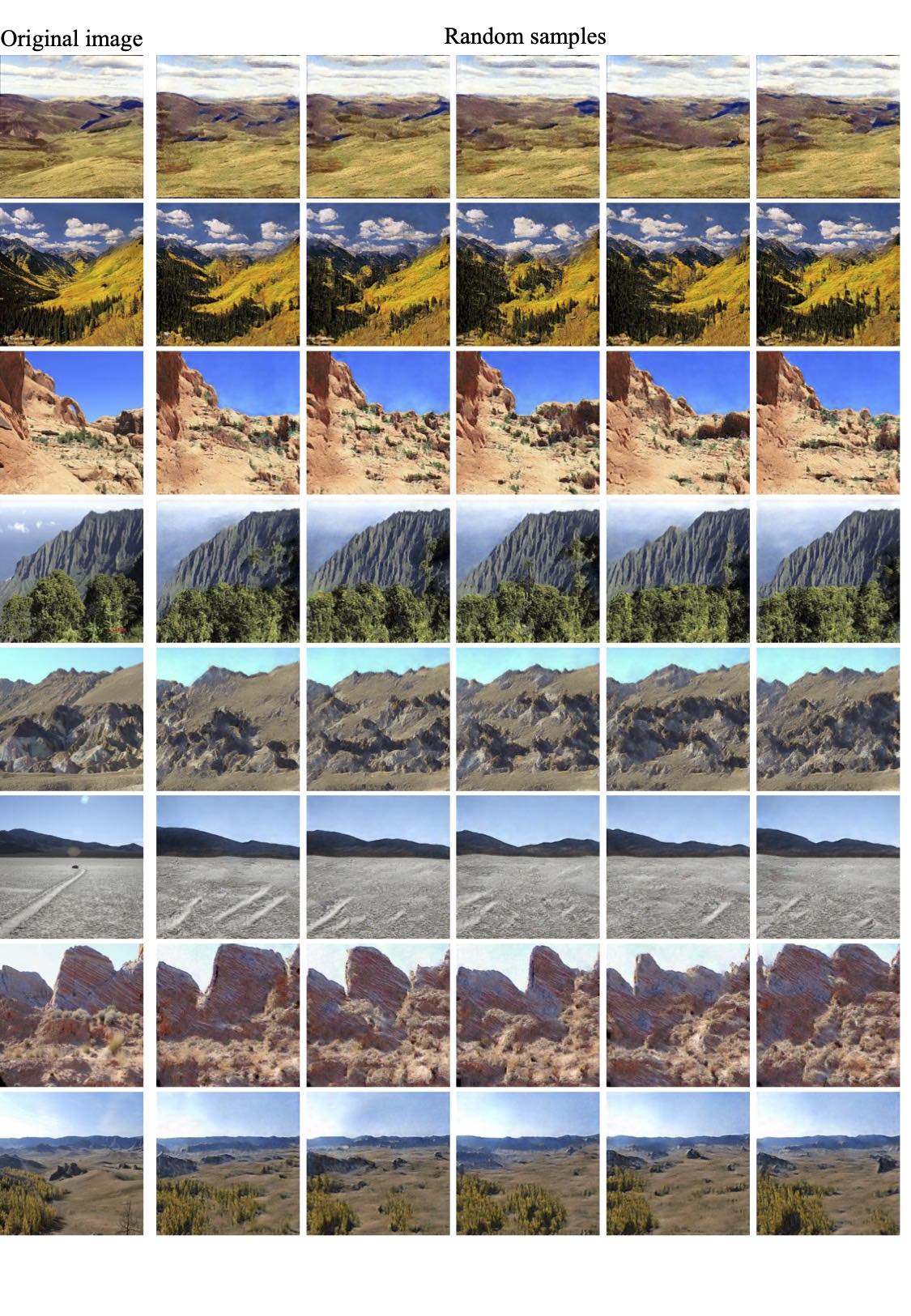}
\newpage

\subsection{LSUN-50}
\includegraphics[width=\linewidth]{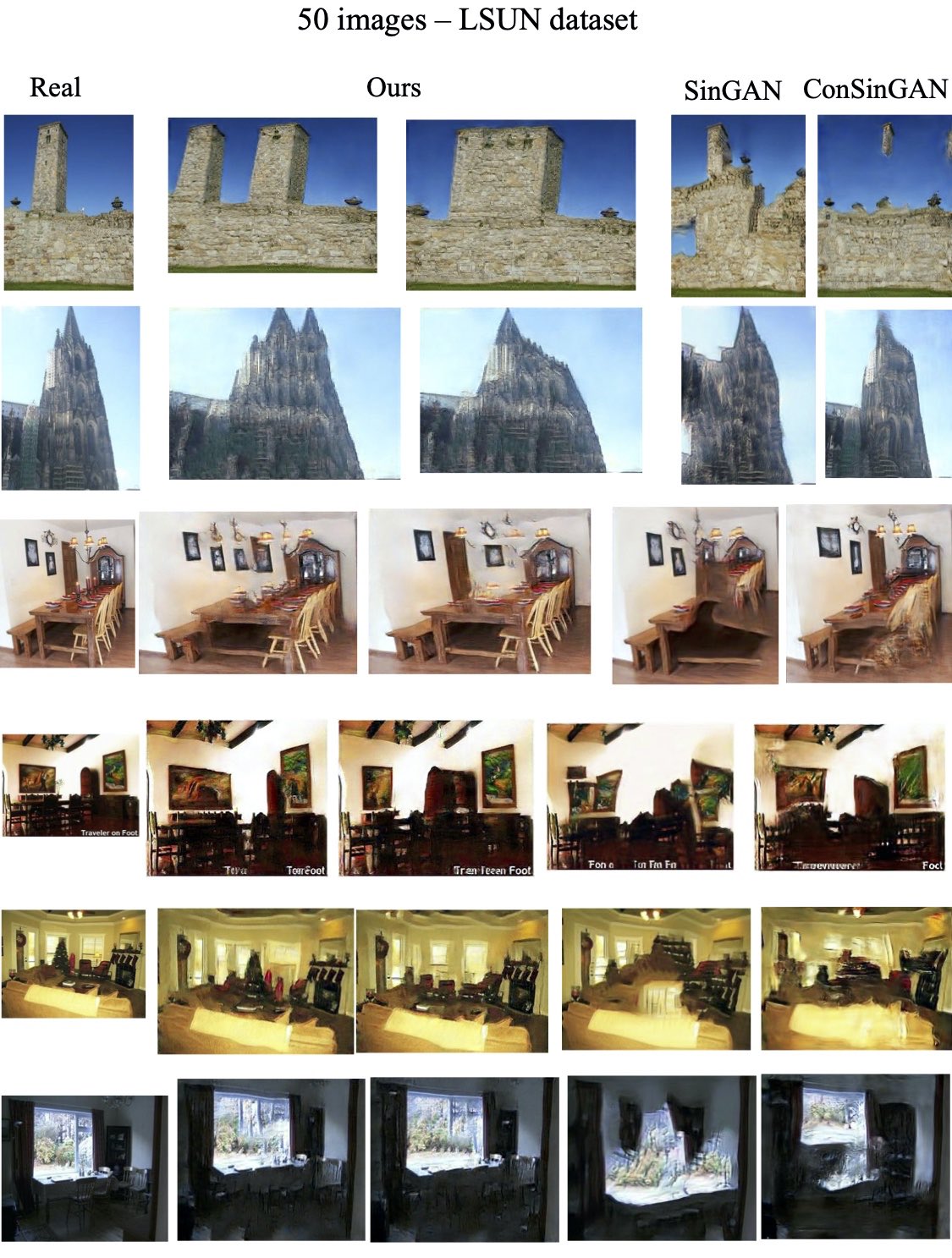}
\clearpage
\newpage

\subsection{C250}
\includegraphics[width=\textwidth]{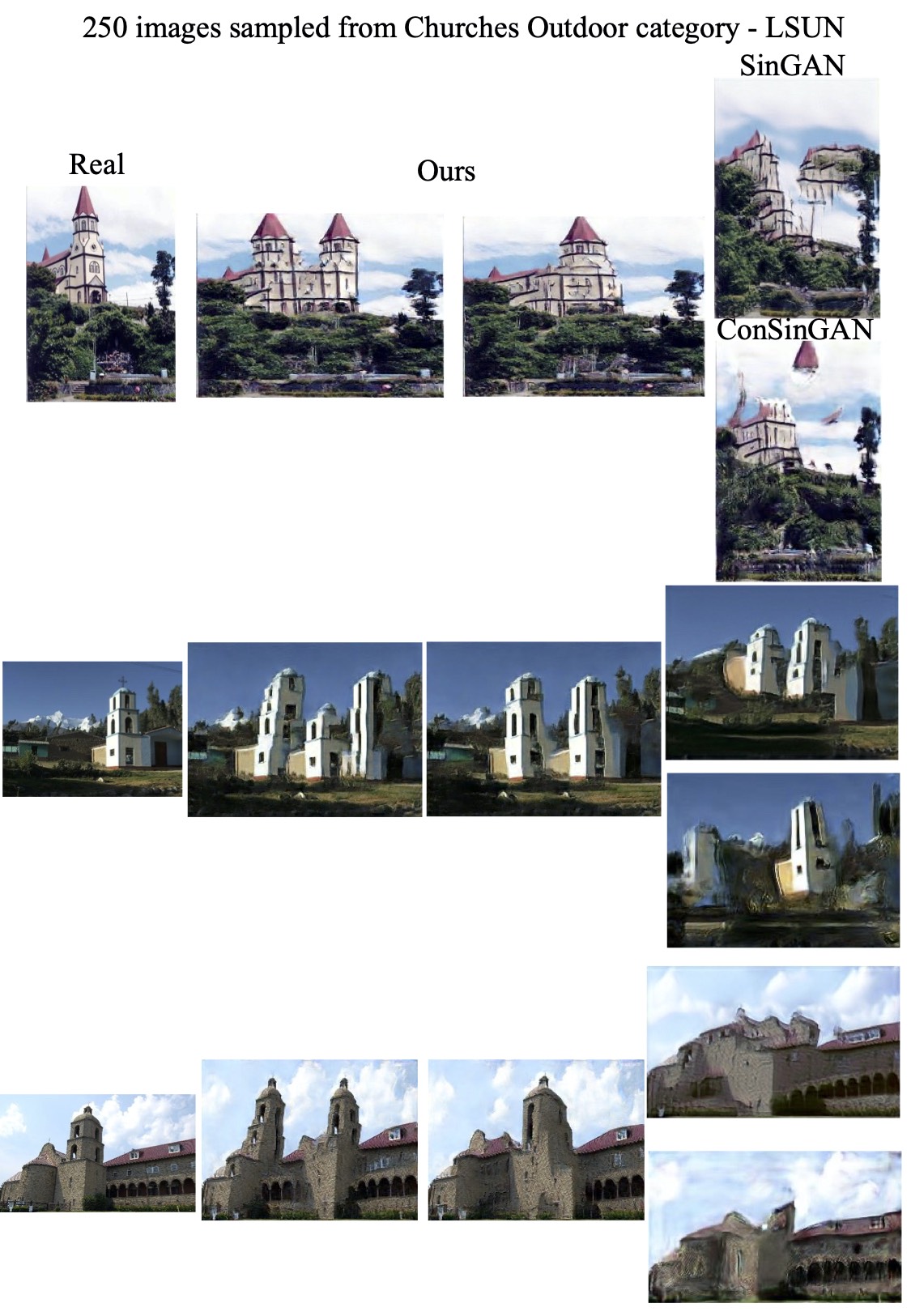}
\clearpage
\newpage

\subsection{CelebA}
We have tested the method on 50 randomly sampled face images from the CelebA dataset. 
We attach side-by-side results along with the baselines, where we used the same initial noise size (of width 22) for all of the methods to allow for a fair comparison.\\
Even though our method generates more realistic images than baselines by a notable margin, our results are still non-comparable to classic face generation (by standard GAN and Flow-based models). Thus, we consider face datasets as a limitation of our approach.

\centering
\includegraphics[width=\linewidth]{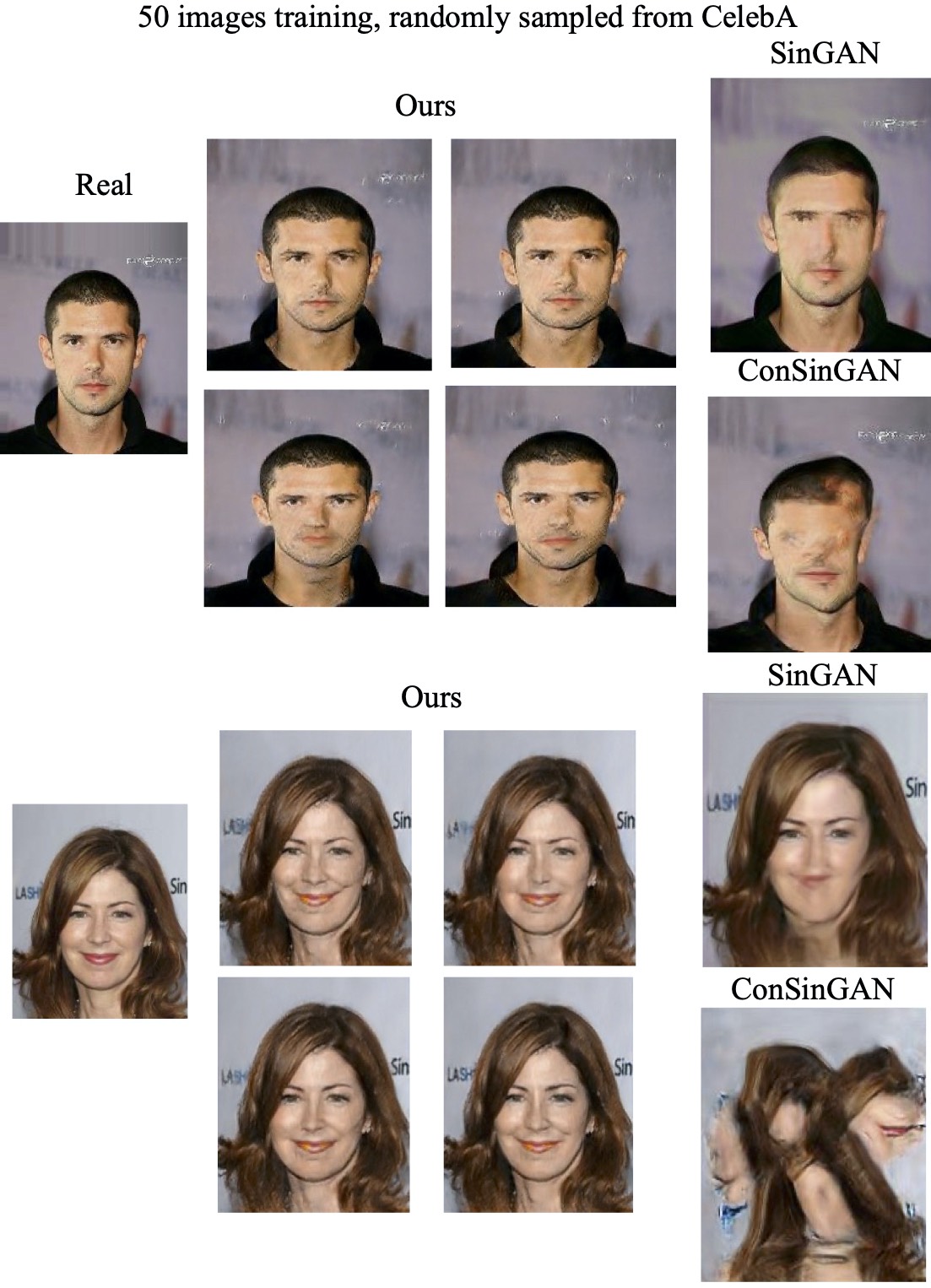}
\clearpage
\newpage
\subsection{Arbitrary sized and aspect-ratio image generation}
Due to the fully convolutional architecture adopted, all of our models are able to generate an image with an arbitrary size of aspect ratio by simply changing the dimensions of the noise maps used. Below are some examples obtained during single mini-batch training and V500 dataset training.
\includegraphics[width=\linewidth]{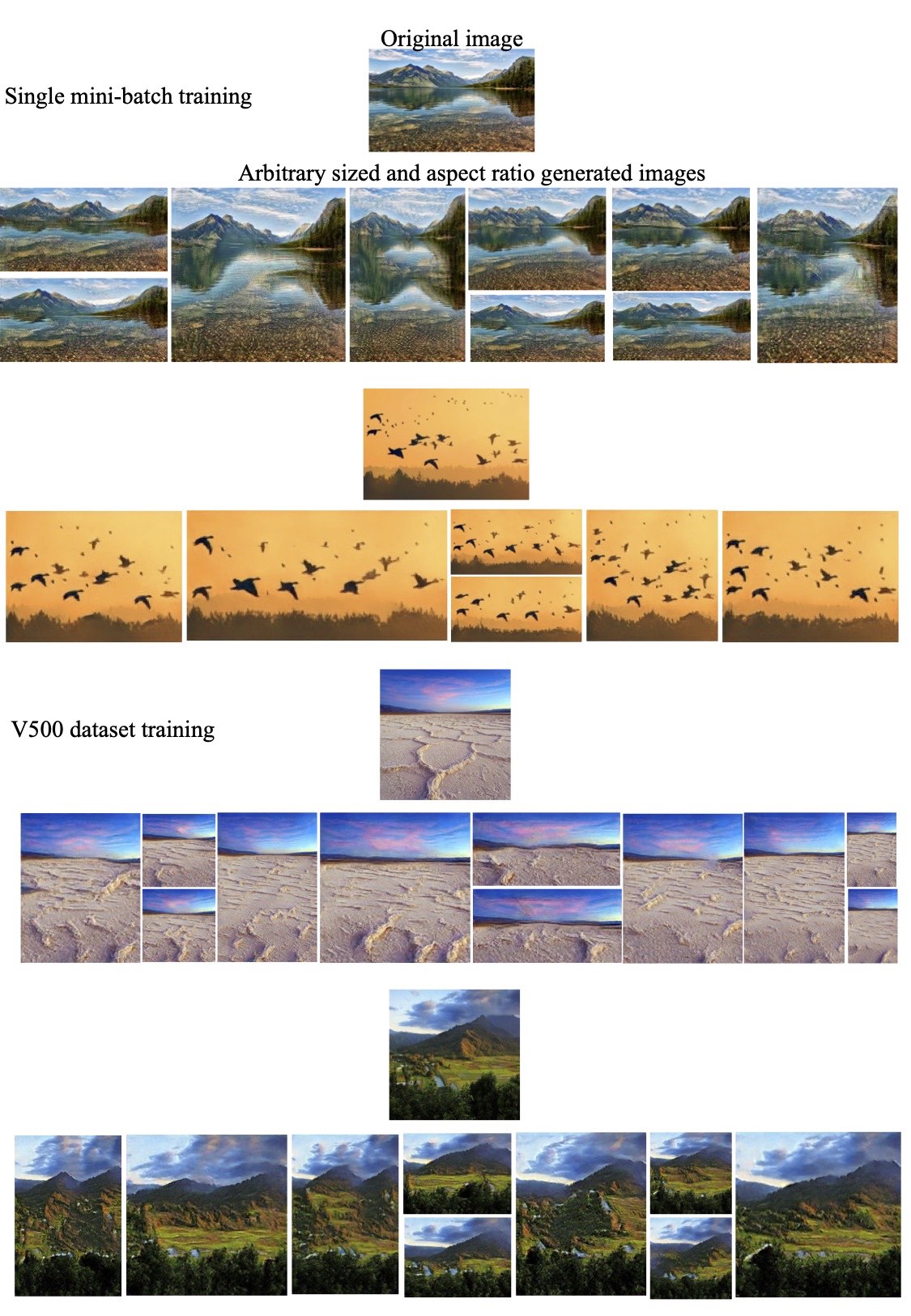}
\clearpage
\newpage

\section{Editing, Harmonization and Animation}
Following results were obtained using a single model, trained on the 50 image dataset merged with these 4 images (a total of 54 images). The applications are performed in the same exact way \cite{shaham2019singan} did. 

\includegraphics[width=\linewidth]{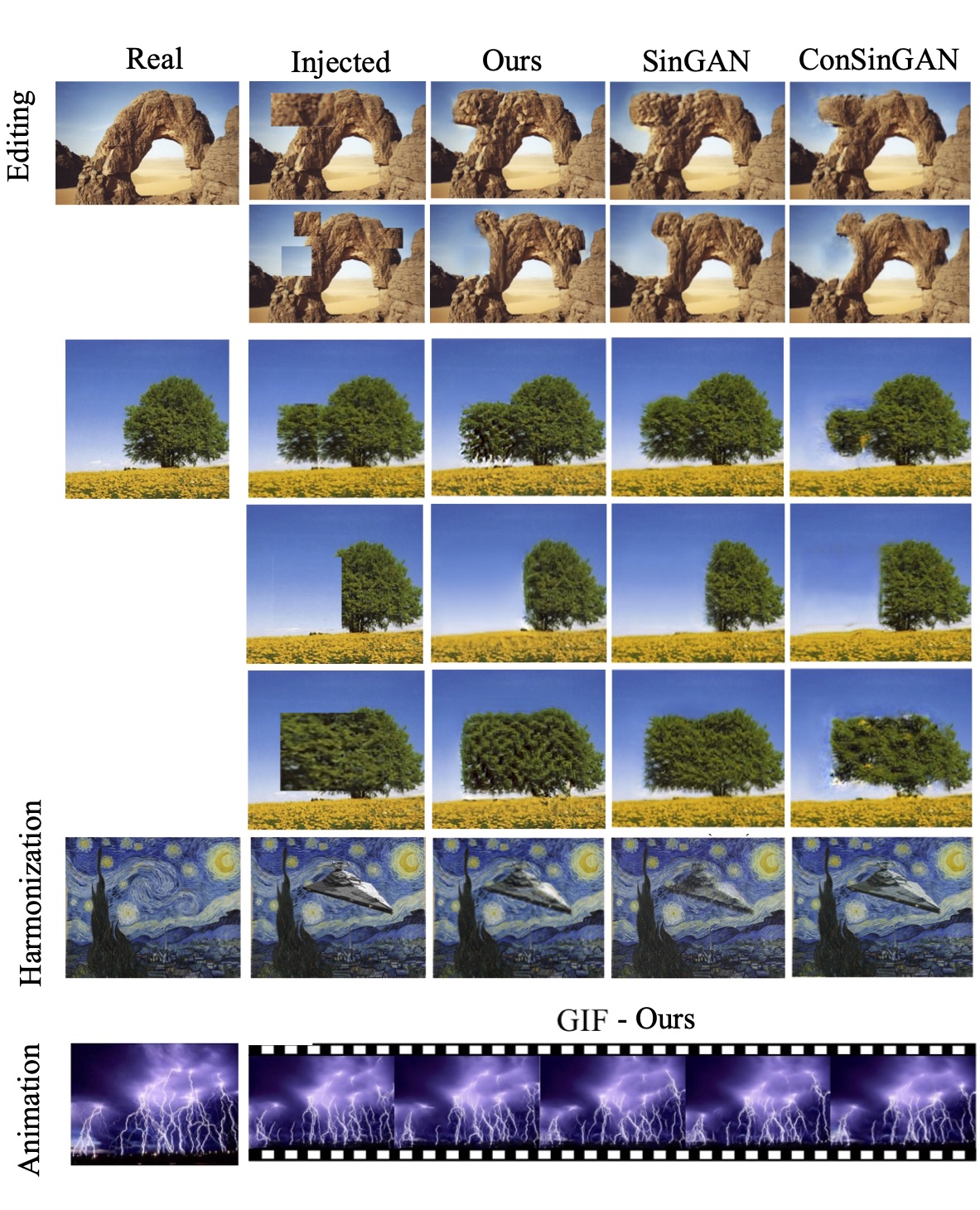}

\clearpage

\section{Interpolation}
We conducted an experiment to study the smoothness of our interpolations at different scales. We estimated the slope of the generated images $H_i(\alpha)$, for a fixed set of random seeds, on a discrete set of values $\alpha \in \{0.1j\}^{9}_{j=1}$  as follows: $s_{i,j} := \frac{\|H_i(\alpha_{j+1}) - H_i(\alpha_j)\|_1}{h \times w \cdot (\alpha_{j+1}-\alpha_j)}$, where $h \times w$ is the size of the images. As can be seen in Fig.~\ref{fig:lip}, the interpolations at higher scales tend to be significantly smoother than the interpolations at lower scales.
\begin{figure}[h]
    \centering
    \setlength{\tabcolsep}{1pt} 
    \renewcommand{\arraystretch}{1} 
    \begin{tabular}{cc}
    \includegraphics[width=0.44\linewidth]{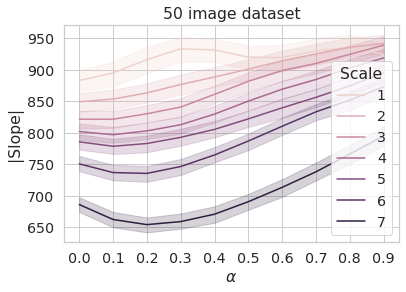}&
    \includegraphics[width=0.44\linewidth]{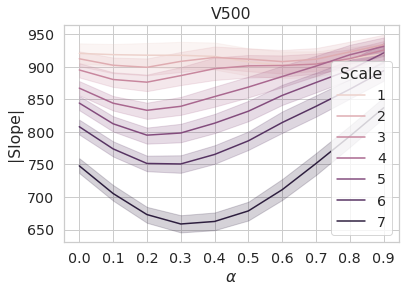}\\
    \includegraphics[width=0.44\linewidth]{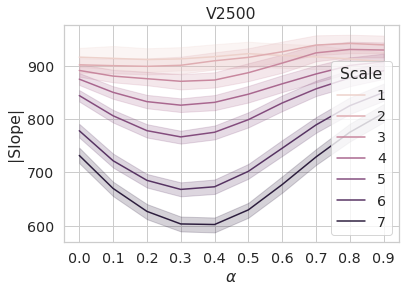}&
    \includegraphics[width=0.44\linewidth]{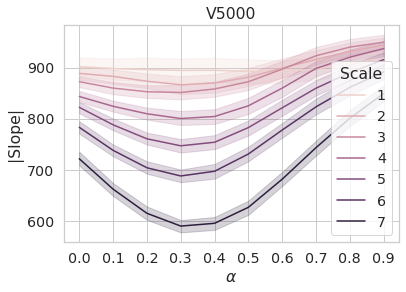}
    \end{tabular}
    \caption{{\bf Smoothness rate of the interpolations.} We plot the smoothness rate $s_{i,j}$ (y-axis) as a function of $\alpha$ (x-axis), averaged over $500$ pairs of images $A,B$ along with their standard deviations.}
    \label{fig:lip}
\end{figure}

\clearpage

\subsection{Places-50}
\includegraphics[width=\linewidth]{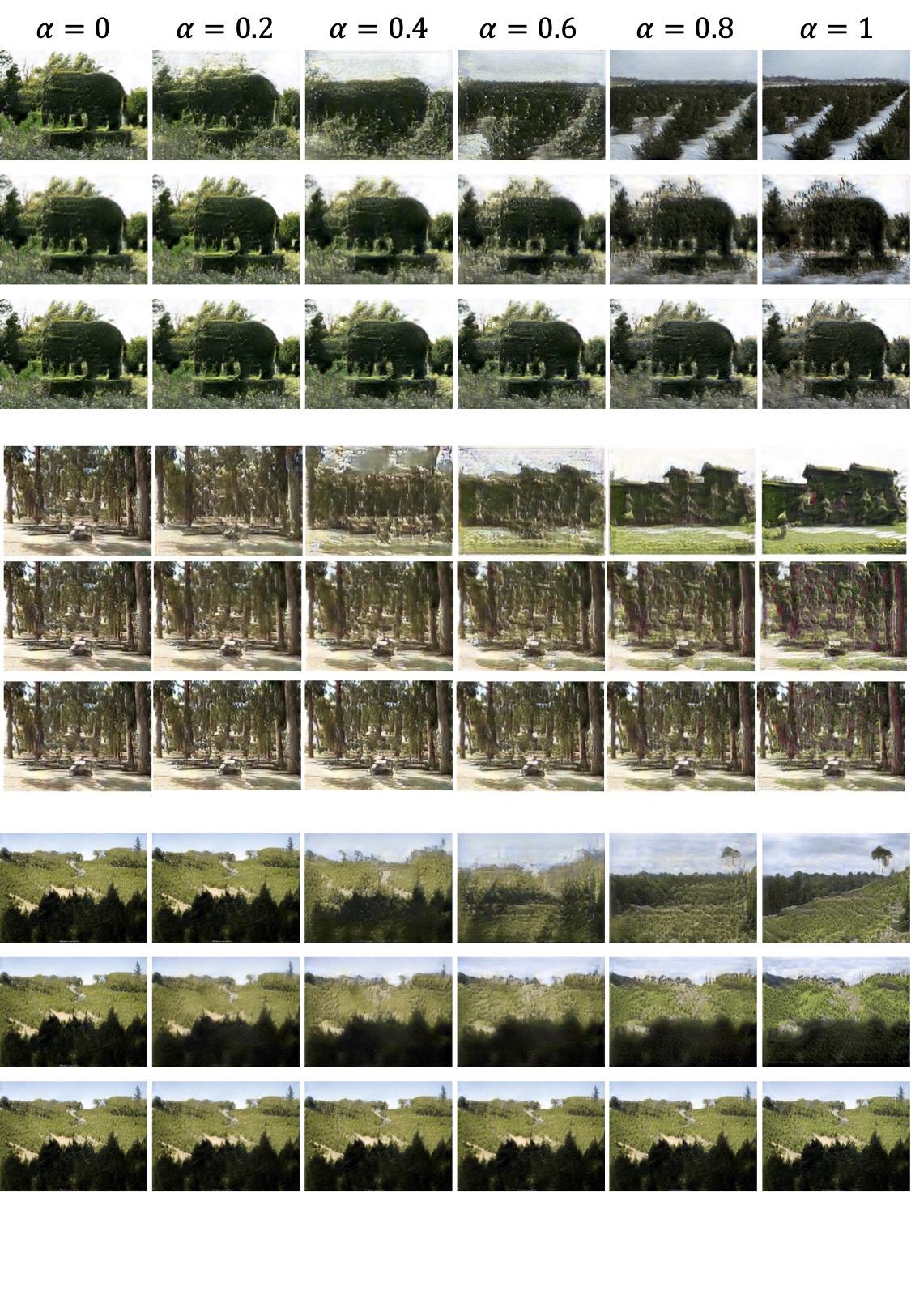} 
\newpage
\subsection{V500}
\includegraphics[width=\linewidth]{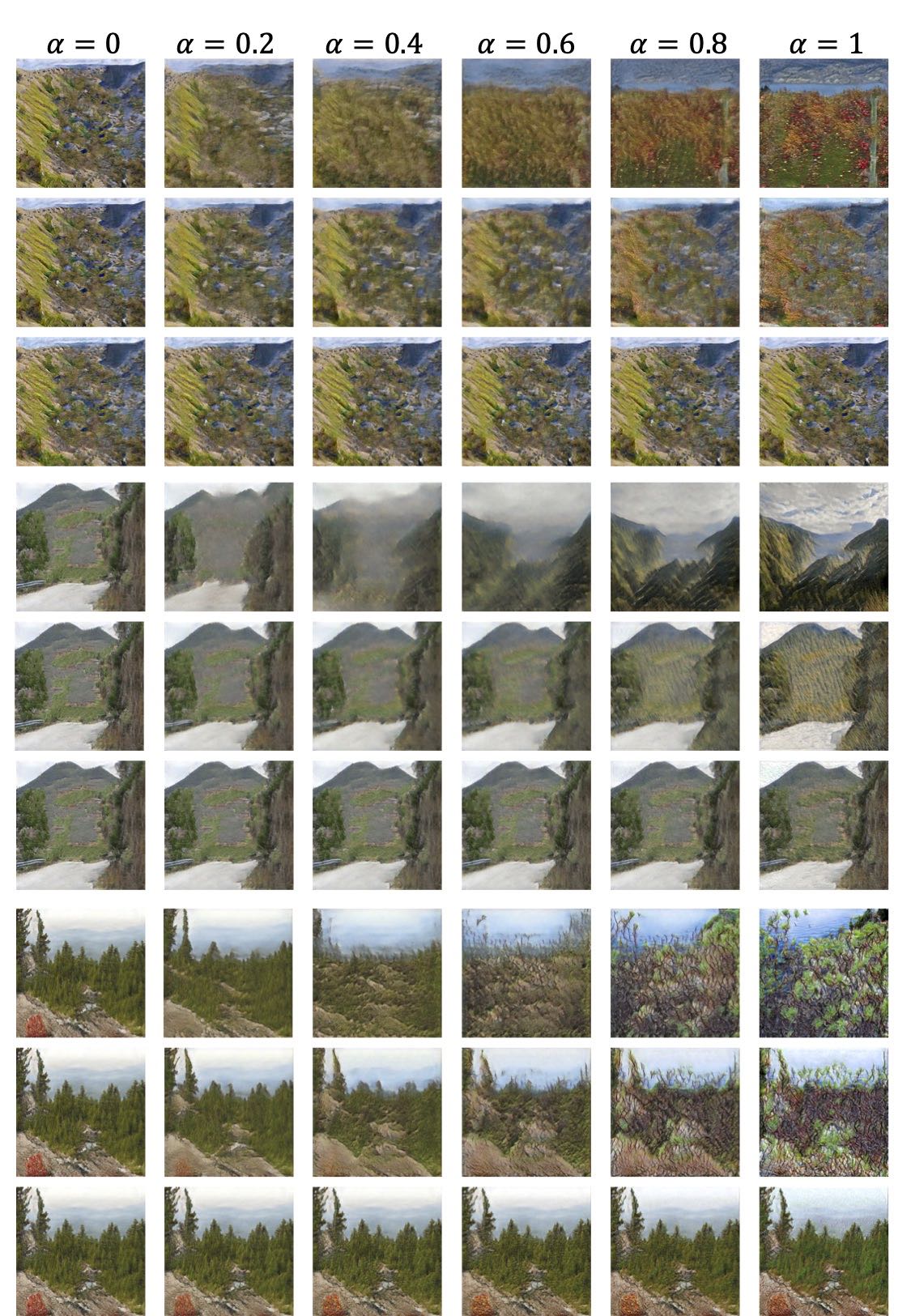} 
\newpage

\subsection{V2500}
\includegraphics[width=\linewidth]{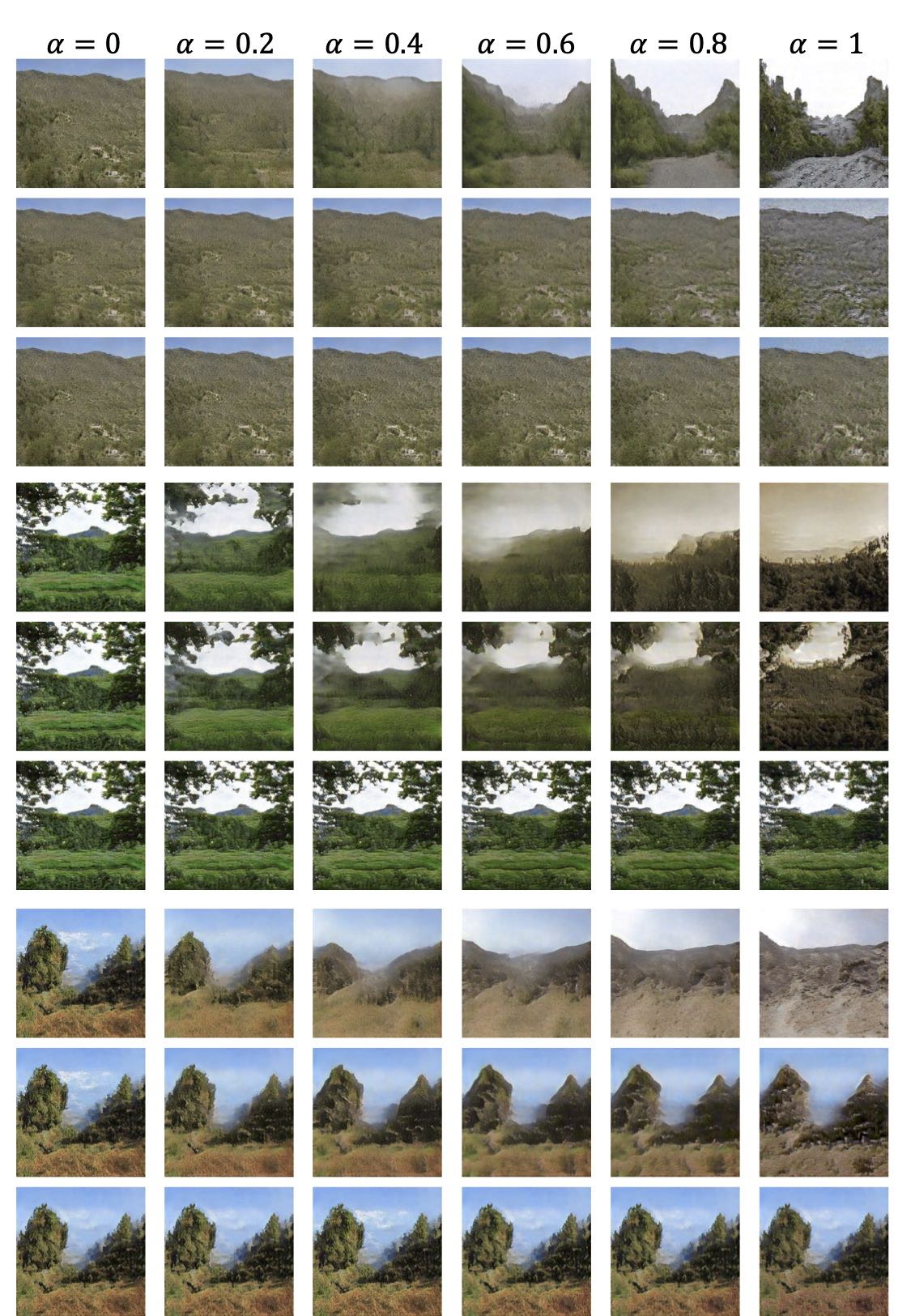} 
\newpage

\subsection{V5000}
\includegraphics[width=\linewidth]{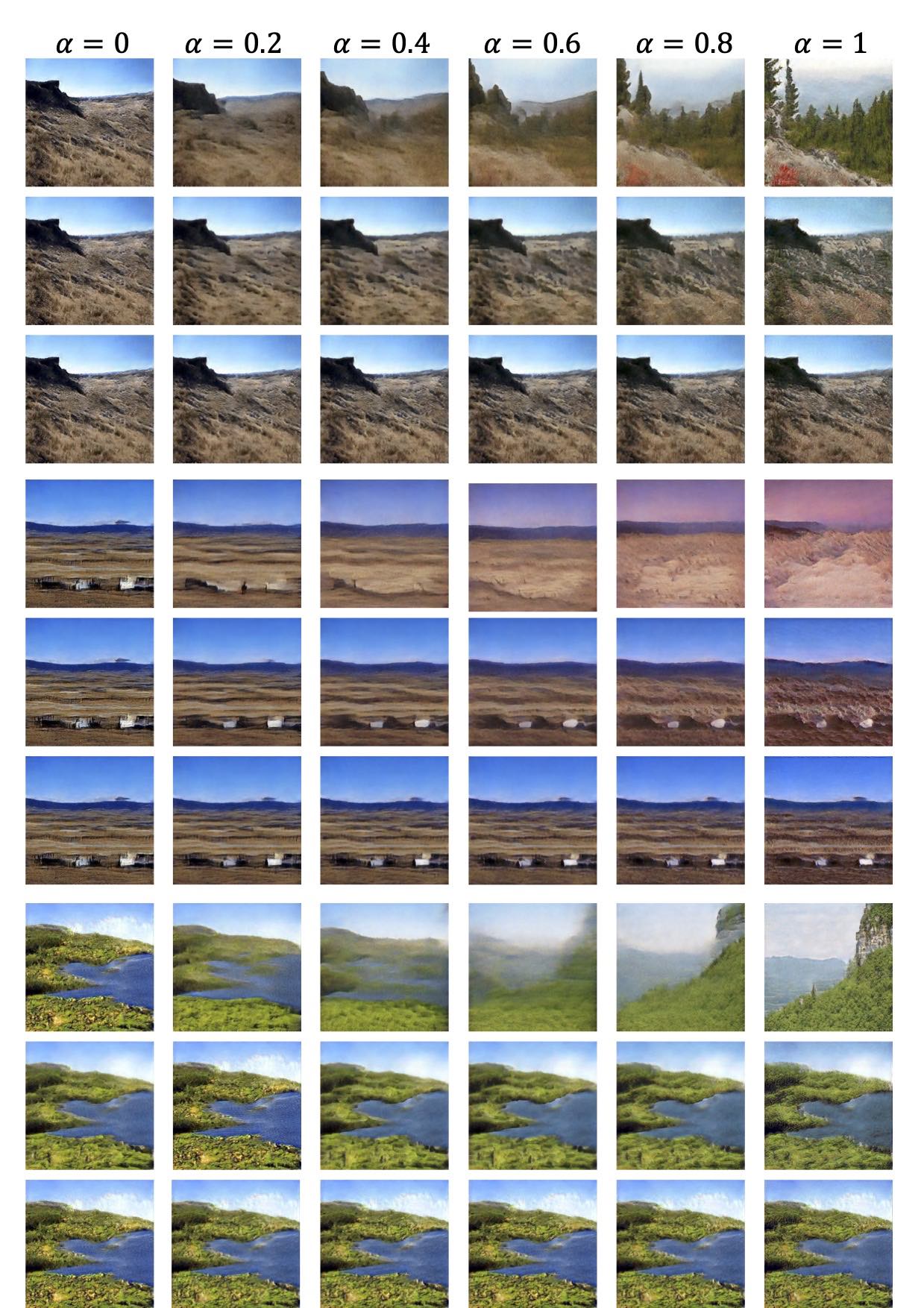} 
\newpage

\section{Feedforward modeling}
\includegraphics[width=\linewidth]{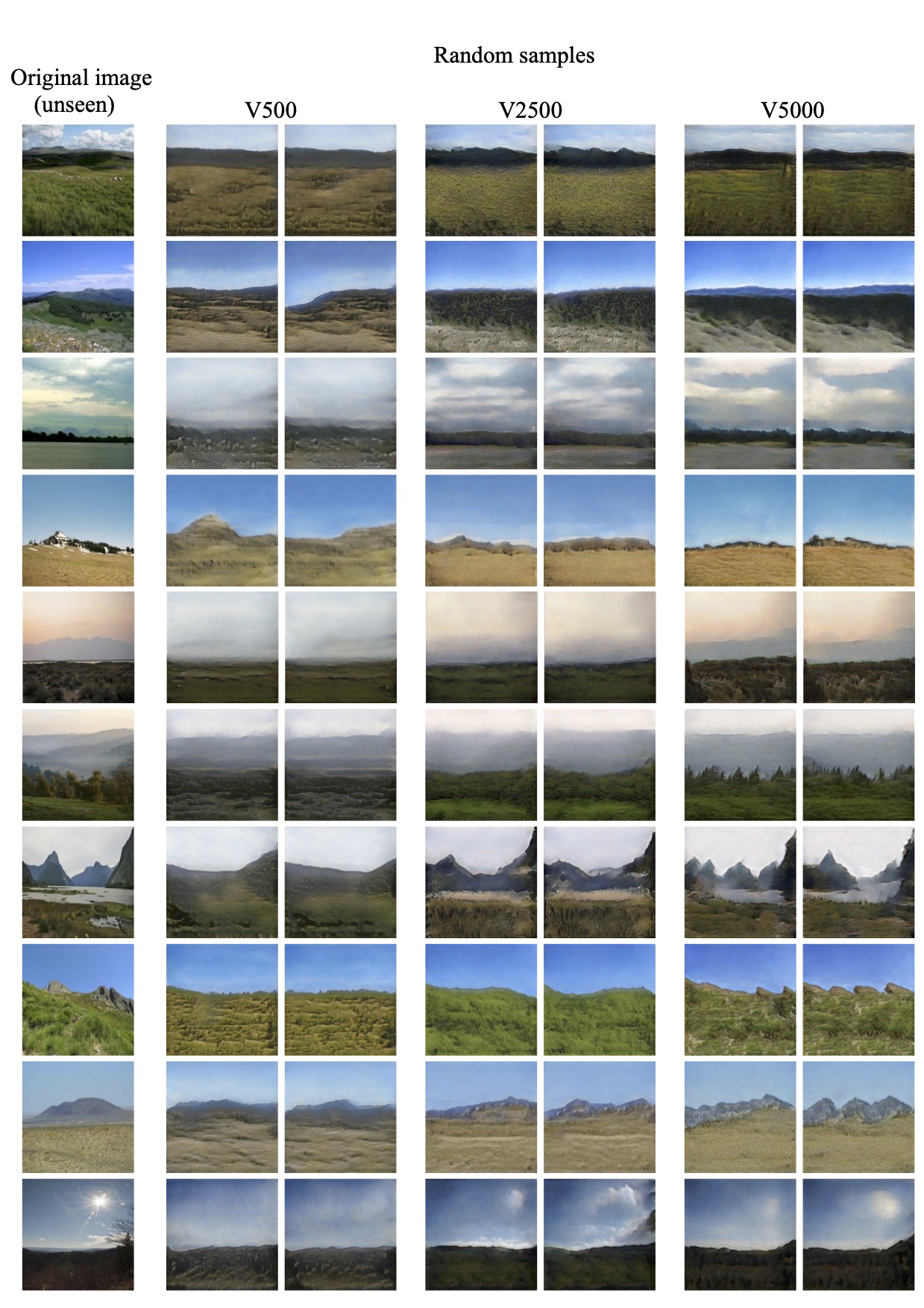}

\clearpage
\end{document}